%% file: main.tex
\documentclass[11pt]{article}
\input{preamble}

\title{The Router Within: Eliciting\\Native Skill Routing from a Frozen LLM}
\runningtitle{GAVEL: The Router Within}
\date{arXiv preprint, September 2026}
\paperlogo{\includegraphics[height=1.5cm]{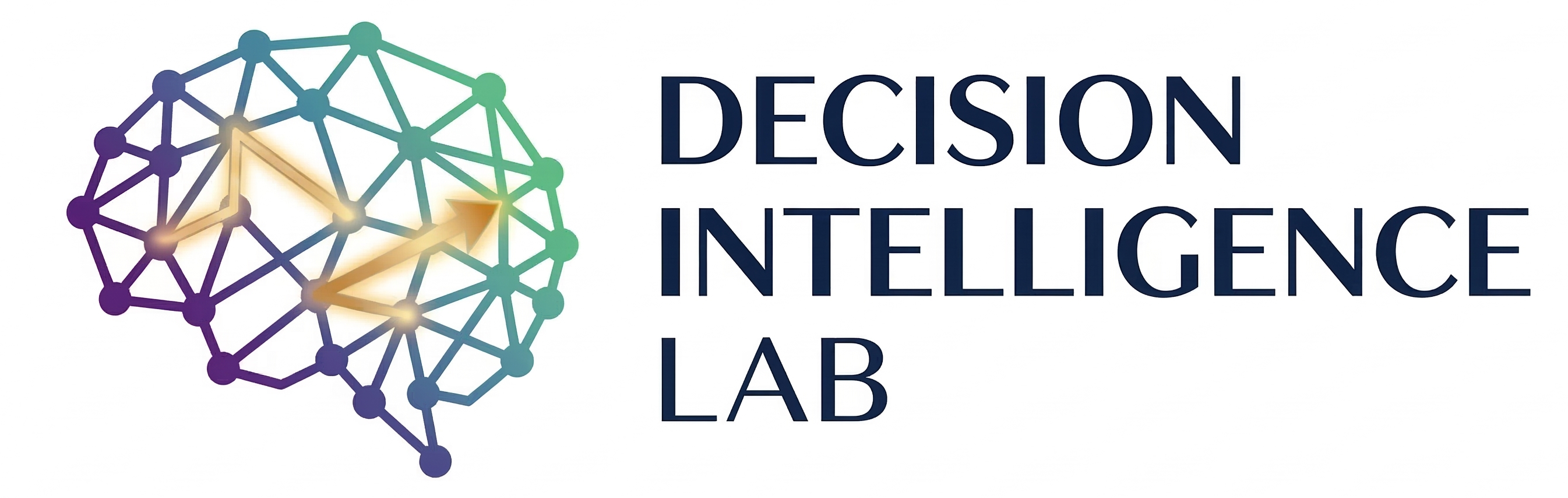}}
\author{
  Ruishuo Chen$^{1}$, Xun Wang$^{1}$, Yu Chen$^{1}$, Zhuoran Li$^{1}$, and
  Longbo Huang$^{1\,\text{\faEnvelope}}$
  \\[0.3em]\normalfont
  $^1$Institute for Interdisciplinary Information Sciences, Tsinghua University
  \\
  \text{\faEnvelope}\ Correspondence: \href{mailto:longbohuang@tsinghua.edu.cn}{\texttt{longbohuang@tsinghua.edu.cn}}
}

\begin{document}
\maketitle
\thispagestyle{fancy}

\input{sections/0_abstract}
\input{sections/1_introduction}
\input{sections/2_related}
\input{sections/3_method}
\input{sections/4_experiments}
\clearpage
\input{sections/5_conclusion}
\FloatBarrier

\begingroup
\small
\urlstyle{same}
\bibliographystyle{dilab_ref}
\bibliography{references}
\endgroup

\makeappendixtoc
\appendix
\input{appendix/01_native_readouts}
\input{appendix/02_layer_selection}
\input{appendix/03_training}
\input{appendix/04_compression}
\input{appendix/05_gemma}
\input{appendix/06_ablations}
\input{appendix/07_prompts}
\input{appendix/08_benchmarks}
\clearpage
\input{appendix/09_deployment}
\input{appendix/10_cost}
\clearpage
\input{appendix/11_discussion}
\end{document}

%% file: preamble.tex
\usepackage{dilab_arxiv}
\input{math_commands.tex}

\usepackage{enumitem}
\usepackage{flafter}
\usepackage{placeins}
\usepackage{bookmark}
\apptocmd{\thebibliography}{\raggedright}{}{}

\newtheorem{proposition}{Proposition}
\newtheorem{corollary}{Corollary}

\setlist[itemize]{leftmargin=*,itemsep=3pt,topsep=5pt}

\hypersetup{
  pdftitle={The Router Within: Eliciting Native Skill Routing from a Frozen LLM},
  pdfauthor={Ruishuo Chen, Xun Wang, Yu Chen, Zhuoran Li, and Longbo Huang},
  pdfsubject={Gavel: skill routing from a frozen language model},
  pdfkeywords={large language models, agents, skills, routing, Gavel},
  bookmarksnumbered=true
}

%% file: math_commands.tex
\usepackage{amsmath,amsfonts,bm}

\def\eqref#1{equation~\ref{#1}}
\def\1{\bm{1}}

\DeclareMathAlphabet{\mathsfit}{\encodingdefault}{\sfdefault}{m}{sl}
\SetMathAlphabet{\mathsfit}{bold}{\encodingdefault}{\sfdefault}{bx}{n}

%% file: sections/0_abstract.tex
\begin{abstract}
Skills extend an LLM agent beyond its parametric knowledge, and the gain they promise rests on picking the right one. Deployed harnesses route by preloading every skill's metadata into the context, which disperses the agent's attention and caps the library size. Retrieval pipelines move the selection out of the context, but also out of the agent's capability. We show that the frozen agent LLM already carries the routing signal in its own forward passes, and that two linear maps suffice to read it out with no skill text in the context. Our \textsc{Gavel} (\textbf{G}lance \textbf{A}nd \textbf{Ve}rdict from a frozen \textbf{L}LM) reads it in two steps. A glance scores the full library by matching the task's mid-layer states against a compact bank that one forward pass builds for each skill at installation, with the two maps as the only trained parameters. A verdict then resumes each shortlisted skill's forward pass, reads the model's own likelihood and yes/no judgment, and fuses both with the glance as a product of experts. Trained once, \textsc{Gavel} transfers zero-shot to three public benchmarks and SkillTraj, our new benchmark of 372 simulated agent trajectories. On Qwen3-32B it outperforms progressive disclosure and retrieve-and-rerank pipelines that add 1.2B to 16B external parameters, by up to 13.4 points on written tasks and up to 21.9 when the need for a skill arises mid-rollout. Routing accuracy improves as the backbone does, and in a bash-agent harness \textsc{Gavel} lets the 32B trigger the right skill on Skill-Use more often than models of up to 1.6T parameters in Codex.
\end{abstract}

%% file: sections/1_introduction.tex
\section{Introduction}
\label{sec:intro}

\begin{figure}[t]
\centering
\includegraphics[width=\textwidth]{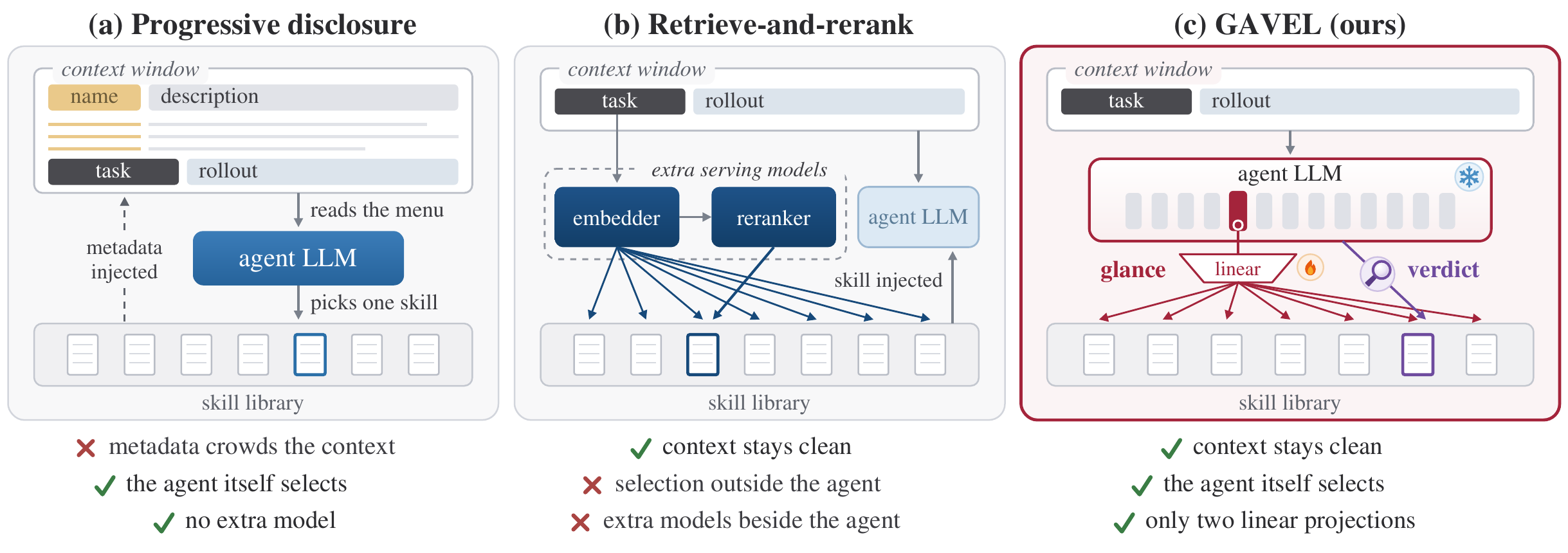}
\caption{Three designs for skill routing. \textbf{(a)}~Progressive
disclosure preloads every skill's metadata into the prompt, crowding the
context. \textbf{(b)}~Retrieve-and-rerank keeps the context clean but
hands selection to external models that lack the agent LLM's capability.
\textbf{(c)}~\textsc{Gavel} routes with the frozen agent LLM itself: a
glance over the whole library, read from one mid-layer state through
two trained linear maps, then a verdict in which the same model examines the
top candidates in full.}
\label{fig:intro}
\end{figure}

Skills have become the standard way to extend an LLM agent beyond its
parametric knowledge \citep{anthropic2025skills,xu2026agentskillssurvey}. A
skill packages instructions, scripts, and reference files in a \texttt{SKILL.md} document that the agent loads into its
context when a task calls for it. A well-chosen skill improves task performance,
while an ill-suited one leaves it worse off than no skill at all
\citep{li2026skillsbench}, and the choice grows harder as public libraries reach
tens of thousands of entries \citep{openai2025codexskills,gao2026skillreducer}.
Selecting the right skill from such a library is therefore the central problem, and
current systems take one of two routes.

Deployed agents, Claude Code and Codex among them, route by progressive
disclosure \citep{anthropic2025skills,openai2025codexskills}: they preload every
installed skill's name and description into the system prompt and let the agent
LLM itself decide which to read in full. The metadata, however, crowds the context
in proportion to the library, degrading the agent's work on the task
\citep{liu2024lost,modarressi2025nolima}. Codex therefore caps skill metadata at
2\% of the context window and trims the rest \citep{openai2025codexskills}.
This confines a deployment to a small library, where routing accuracy still decays
logarithmically with the number of skills \citep{chen2026scalinglaws}. The cost
is intrinsic, since routing inside the context spreads the agent's attention
over every skill it considers and judges each by a summary that omits
much of what selection depends on \citep{zheng2026skillrouter}.

To lift this burden from the context, recent skill routers, following tool
learning \citep{qin2024toolllm,zheng2026skillrouter,wang2026r3skill}, hand the
selection to external models. A retrieval pipeline, usually an embedding model
and a reranker, selects outside the context and injects only the chosen skill.
The context is spared, but the selection is also cut off from the agent's capability. A
standalone retriever may excel at matching text, but judges what
a task needs less reliably than the agent LLM
\citep{shi2025toolret,cho2026skillret}, especially amid the noisy context where
the need arises mid-rollout \citep{lumer2025scalemcp,fei2025mcpzero}. It also
does not improve as the agent itself does.

Selection inside the context thus taxes the agent, yet an external model does
not understand tasks and skills as well as the agent does. Can a router, then, select with
the agent's capability but outside its context? The open question is through
what interface, and at what cost, that understanding can be turned to routing.
We propose \textsc{Gavel} (\textbf{G}lance \textbf{A}nd \textbf{Ve}rdict from a
frozen \textbf{L}LM), the third design in Figure~\ref{fig:intro}, which draws
every routing signal from the frozen agent LLM's own forward passes with only two
linear projections and lets no skill text into the context until a
skill is chosen.

\textsc{Gavel} routes in two stages, a glance over the whole library followed
by a verdict on its shortlist. As the LLM decodes, its intermediate layers
compress the input's semantics \citep{skean2025layers}, but entangle them with
much that serves only the next token \citep{queipo2026sinks}, so no native
read-out ranks a library's skills (Appendix~\ref{app:native}). The glance
therefore grafts a new attention head onto the model. A query map reads each
task token's state at a mid layer, and a key map reads each skill token's state
there at installation. Both are trained contrastively so that a task token's
strongest match lands on a skill that serves it. Each task token votes for the
skills it matches best, so its few decisive tokens are not averaged away. A new
skill thus costs one forward pass and no training.

A factorized head, however, misses subtler inferences routing can turn on.
Full attention over skill and task does not, but it costs
one forward pass per candidate, which no library affords, so the verdict spends
it on the glance's shortlist alone. For each it resumes the installation pass with the
task appended, and reads from that forward both how strongly the skill primes the
model for the task and its log-odds judgment of whether the skill serves it.
Unlike a pipeline's scores, the three signals can each be read as the same log
posterior of skill given task, the glance contrastively, the likelihood
generatively, the judgment discriminatively, so they fuse as a product of
experts \citep{hinton2002products}.

We instantiate \textsc{Gavel} with Qwen3-32B as the agent LLM and train its
two projections, 7.9M parameters in all, once on the synthetic query--skill pairs
of SkillRet \citep{cho2026skillret}. We then evaluate it on three public
skill-selection benchmarks, which depart from that corpus in query author,
document genre, and library scale. All three hand the router a written task,
whereas a live agent's need often surfaces mid-rollout. We therefore also build
SkillTraj, a benchmark of 372 simulated agent trajectories that scores each
router at that moment under four scenarios of noisy multi-turn context. Across all
four, \textsc{Gavel} outperforms progressive disclosure given a
well-chosen shortlist of twenty skills and retrieve-and-rerank
pipelines that add 1.2B to 16B external parameters. It leads the strongest pipeline by 13.4 points on
SRA-Bench and by 8.6 to 21.9 points across SkillTraj's four scenarios.
Replacing any design choice costs accuracy on every benchmark, and
routing improves as the backbone does. In a bash-agent harness, \textsc{Gavel} lets Qwen3-32B trigger the
correct skill on Skill-Use \citep{han2026skilluse} more often than
frontier models of up to 1.6T parameters in Codex.

In summary, our contributions are as follows:
\begin{itemize}[leftmargin=1.5em, itemsep=0pt, parsep=0pt]
\item We propose \textsc{Gavel} (\textbf{G}lance \textbf{A}nd \textbf{Ve}rdict from a frozen \textbf{L}LM), a skill
router that elicits routing from the frozen agent LLM's own forward passes
and thus improves as the agent does, with no model beside it and no skill
text in the context until one is chosen.
\item On the technical side, a token-level glance reads a mid layer through
two linear maps against skill banks thinned to $\varepsilon$-covers with
bounded distortion, and a verdict reads generative and discriminative evidence
from a resumed forward pass. All three are fused as a product of experts.
\item We build SkillTraj, a benchmark of 372 simulated agent
trajectories that scores a router at the moment a skill becomes
needed, under four scenarios of noisy multi-turn context.
\item Across three public benchmarks and SkillTraj, \textsc{Gavel}
outperforms progressive disclosure and retrieve-and-rerank pipelines
that add up to 16B external parameters, by up to 21.9 points, and in
a live harness has Qwen3-32B load the correct skill more often than far
larger frontier models.
\end{itemize}

%% file: sections/2_related.tex
\section{Related Work}
\label{sec:related}

\textbf{Skill and tool selection.\,\,}
Deployed harnesses route by progressive disclosure, holding every skill's
metadata in context for the agent to read
\citep{anthropic2025skills,openai2025codexskills}. Retrieval-based systems move
the selection to an external stack, built over API pools
\citep{qin2024toolllm,lumer2025scalemcp}, invoked mid-rollout
\citep{fei2025mcpzero}, or trained on full skill documents
\citep{zheng2026skillrouter,wang2026r3skill}. ToolkenGPT
\citep{hao2023toolkengpt} has a frozen LLM emit a tool as a token, but
trains one embedding per tool, so a new tool costs a training run.
\textsc{Gavel} instead draws the routing signal from the agent's own
forward passes, with no skill text in the context and no model beside
the backbone.

\textbf{Reading information out of an LLM.\,\,}
Probing shows that hidden states carry more than the next token needs
\citep{alain2017understanding,belinkov2022probing}.
Text-embedding work scales this read-out, turning decoder LLMs into strong
encoders by fine-tuning \citep{wang2024improving,behnamghader2024llm2vec} or
by prompting alone \citep{springer2025repetition}. GRIT
\citep{muennighoff2025generative} unifies embedding and generation in one
model, but fine-tunes the whole backbone. Rerankers read the model's
predictions just as directly, scoring a document by the likelihood it
assigns to the query \citep{sachan2022improving,zhuang2023open} or by
its own yes/no relevance judgment
\citep{nogueira2020document,sun2023chatgpt}. \textsc{Gavel} turns these
read-outs into a skill router on the frozen agent LLM itself.

%% file: sections/3_method.tex
\section{\textsc{Gavel}: Routing on the Agent's Own Forward Passes}
\label{sec:method}

\subsection{Problem setup and design principles}
\label{sec:setup}

We model skill routing as follows. A skill library
$\mathcal{S}=\{s_1,\dots,s_N\}$ is a set of documents, each a metadata
header (name and description) followed by a body of instructions. The agent
is a frozen LLM $\mathcal{M}$ decoding a rollout. At a routing point, its
context $x$ holds the task, either a user request or an execution state
reached mid-rollout. A router observes $x$ and $\mathcal{S}$ and decides
which skill to load.

Current practice, as Section~\ref{sec:intro} reviewed it, gives rise to three
requirements on a router. (i) \emph{Task-only context}. Skill metadata preloaded into the prompt
disperses the agent's attention and hurts its work on the task
\citep{liu2024lost,modarressi2025nolima}, so no skill text may enter the context
until one is chosen. (ii) \emph{No standalone model}. The agent LLM judges a
task's needs more reliably than a standalone retriever
\citep{shi2025toolret,cho2026skillret}, so selection should inherit that
capability and improve as the agent does, with anything added on top
lightweight. (iii) \emph{Installation-time indexing}. Skills are files dropped
in a folder, and public libraries hold tens of thousands that keep changing
\citep{gao2026skillreducer}, so no training may be run for a new skill.

\textsc{Gavel} satisfies all three (Figure~\ref{fig:method}). The glance
(Sections~\ref{sec:glance} and~\ref{sec:compress}) ranks the full
library by matching hidden states of $\mathcal{M}$ on the task against
those that a single forward pass extracts per skill at installation, with no
training run. The verdict
(Section~\ref{sec:verdict}) re-examines the shortlist with a few forward
passes of $\mathcal{M}$ over skill and task together, and
Section~\ref{sec:fusion} fuses their three read-outs as a product of
experts. Throughout, $h_\ell(c)$ denotes the layer-$\ell$
hidden states of $\mathcal{M}$ on input $c$.

\begin{figure}[t]
\centering
\includegraphics[width=\textwidth]{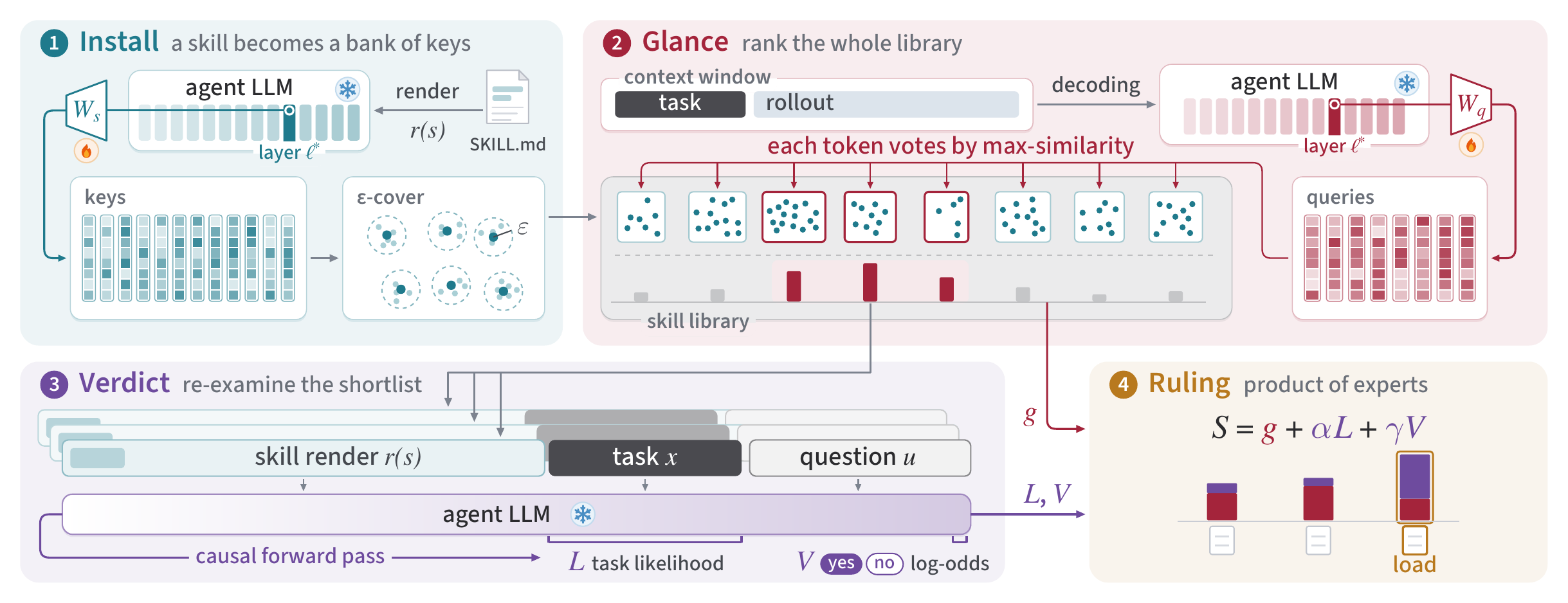}
\caption{The \textsc{Gavel} pipeline. \emph{Install}: $W_s$ maps a new
skill's layer-$\ell^\ast$ states in the frozen agent LLM to a bank of keys
thinned by an $\varepsilon$-cover. \emph{Glance}: $W_q$ reads the same
layer as the model decodes; each task token votes by max-similarity over
every bank, ranking the whole library. \emph{Verdict}: the model
re-examines each shortlisted skill against the task, giving the task
likelihood $L$ and the yes/no log-odds $V$. \emph{Ruling}: a product of
experts over glance and verdict picks the skill to load.}
\label{fig:method}
\end{figure}

\subsection{The glance: token-level read-out of the frozen forward pass}
\label{sec:glance}

The routing signal has to come out of the hidden states of $\mathcal{M}$,
which compress the input into its semantics \citep{skean2025layers} but entangle
it with much that serves only the next token \citep{queipo2026sinks}, so no
native read-out ranks a library (Appendix~\ref{app:native}).

The glance therefore grafts a new attention head onto the frozen
$\mathcal{M}$, a trained query map $W_q$ and key map $W_s$ with no value
path. Both maps read one intermediate layer $\ell^\ast$, chosen at the
floor of the model's compression valley
\citep{skean2025layers,queipo2026sinks}, where the matrix entropy of
the hidden states bottoms out. On our backbone this floor lies roughly
70\% of the way through. Training the read-out at every depth confirms this
unsupervised criterion (Appendix~\ref{app:layer}).

Intuitively, $W_q$ reads out of a task token's state the
capability it asks for, and $W_s$ reads out of a skill token's state the
capability it offers. The head costs little on the query side, since the
states $h_{\ell^\ast}(x)_i$ of the task tokens are a byproduct of the
decoding $\mathcal{M}$ performs anyway. Each becomes a unit query
$q_i \propto W_q\, h_{\ell^\ast}(x)_i$. On the skill side, when a skill
$s$ is installed we render it once in a fixed prompt $r(s)$
(Appendix~\ref{app:prompts}), run one
forward pass, and store a unit key
$d^s_j \propto W_s\, h_{\ell^\ast}(r(s))_j$ for every token $j$ of the
metadata header and body. The head attends across sequences, from a live
context into documents encoded elsewhere, and so carries no positional
encoding. The resulting bank $\{d^s_j\}_j$ is the head's key cache for
the skill, filled by that single forward pass.

At a routing point the head performs hard attention into this bank. Each
task token keeps its strongest logit against every skill, as in
late-interaction retrieval \citep{khattab2020colbert},
\begin{equation}
\label{eq:maxsim}
m_i(s) = \max_j \,\langle q_i, d^s_j\rangle .
\end{equation}
We train $W_q$ and $W_s$ once, contrastively, with the span mean
$\bar{m}(s) = \operatorname{mean}_i\, m_i(s)$ as the match score. Each
training task carries a set $\mathcal{S}^{+}$ of gold skills, negatives
$\mathcal{S}^{-}$ are drawn from the rest of the library, and the loss is the
multi-positive contrastive
\begin{equation}
\mathcal{L} \;=\; -\log
\frac{\sum_{s \in \mathcal{S}^{+}} e^{\tau \bar{m}(s)}}
     {\sum_{s \in \mathcal{S}^{+} \cup \mathcal{S}^{-}} e^{\tau \bar{m}(s)}}
\end{equation}
at temperature $\tau$. Gradients stop at $h_{\ell^\ast}$, so training
updates only the two maps and $\mathcal{M}$ stays frozen.

At routing time, however, we do not read the head out by that mean. Only
a handful of task tokens point at the right skill, and as the span grows,
the rest, each somewhat similar to every skill, bury the handful by sheer
number. Instead, each token votes, contributing $m_i(s)$ only to its
top-$k$ skills:
\begin{equation}
\label{eq:vote}
g(s|x) \;=\; \frac{1}{\sum_i w_i}\sum_{i \,:\, s \in \mathcal{T}_i} w_i\, m_i(s),
\qquad
\mathcal{T}_i = \text{top-}k \text{ skills under } m_i .
\end{equation}
The normalization keeps $g$ on the scale of the mean the head was trained
with, where $w_i$ is an adjustable decay that favors recent tokens, and
$k = \min(10, \max(3, \mathrm{round}(0.1N)))$ is 10 for any library of at least 100 skills.
Glancing the whole library
is thus one sweep of the head over the bank.

\subsection{Compressing the skill bank to an \texorpdfstring{$\varepsilon$}{epsilon}-cover}
\label{sec:compress}

The glance trains nothing at installation, but it stores a key per
token, so the bank grows with the library's total length. Much of it is
redundant. A document that dwells on one capability leaves a cluster of
keys pointing nearly the same way, and since every read-out passes
through the max of Eq.~(\ref{eq:maxsim}), a single representative serves
the whole cluster. At installation we therefore keep only a subset
$C_s \subseteq D_s = \{d^s_j\}_j$ with every discarded key within
distance $\varepsilon$ of a retained one.

\begin{proposition}[Distortion, proved in Appendix~\ref{app:compress}]
\label{prop:distortion}
Let $C_s \subseteq D_s$ be an $\varepsilon$-cover of the unit-norm bank
$D_s$ in Euclidean distance. Then for every unit vector $q$,
\begin{equation*}
\max_{d \in D_s}\,\langle q, d\rangle - \varepsilon
\;\le\; \max_{c \in C_s}\,\langle q, c\rangle
\;\le\; \max_{d \in D_s}\,\langle q, d\rangle .
\end{equation*}
\end{proposition}

Compression thus lowers each score by at most $\varepsilon$ and inflates
none, and the vote of Eq.~(\ref{eq:vote}) inherits the bound for every
skill clear of the top-$k$ cuts (Appendix~\ref{app:compress}).

We build the cover by a farthest-first traversal of a skill's keys
\citep{gonzalez1985clustering}, repeatedly retaining the key farthest
from those already retained until every remaining key lies within
$\varepsilon$ of them.

The retained set is also $\varepsilon$-separated, so its size is
pinned between the covering and packing numbers of $D_s$ at scale
$\varepsilon$ (Proposition~\ref{prop:size}, Appendix~\ref{app:compress}).
A skill's bank is thus sized by the number of
$\varepsilon$-distinguishable directions its tokens span rather than by
its length. Since
compression coarsens the geometry the query map was trained against, we
let $W_q$ take one short fine-tune against the training library's
compressed banks under the same loss, with $W_s$ and the keys frozen.

\subsection{The verdict: native signals under the model's full attention}
\label{sec:verdict}

The single factorized head that makes the glance cheap also denies it
the subtler semantic inferences on which a routing decision can turn.
Once the library is cut to a shortlist, we can afford the frozen model's
full attention over skill and task together. The verdict therefore appends the task to the render $r(s)$ encoded at
installation and reads its signals from the predictive distribution
$p_{\mathcal{M}}$ of one forward pass over this continuation. This follows
the order of classical query-likelihood scoring \citep{ponte1998language}
and leaves the installation pass a reusable prefix.

As the pass advances over the task, the predictive distribution at each
position scores the task token that actually comes next, with the skill
now in the prefix. The mean log-likelihood of the task,
\begin{equation}
\label{eq:lchan}
L(s|x) \;=\; \frac{1}{|x|}\sum_{i=1}^{|x|} \log p_{\mathcal{M}}(x_i \,|\, r(s),\, x_{<i}),
\end{equation}
measures how well the skill anticipates the task
\citep{sachan2022improving}. For the second
read-out we close the continuation with a fixed question $u$ asking
whether this skill provides what the task needs
(Appendix~\ref{app:prompts}), and take the model's log-odds of yes over no at the final
position,
\begin{equation}
\label{eq:vchan}
V(s|x) \;=\; \log\!\sum_{t \in \mathcal{Y}} p_{\mathcal{M}}(t \,|\, r(s), x, u)
\;-\; \log\!\sum_{t \in \mathcal{N}} p_{\mathcal{M}}(t \,|\, r(s), x, u),
\end{equation}
where $\mathcal{Y}$ and $\mathcal{N}$ collect the spellings of yes and
no \citep{nogueira2020document}. Since $u$ follows the task, the causal
mask leaves every task position untouched, so one pass yields both
read-outs.

\subsection{The ruling: a product of experts}
\label{sec:fusion}

In a retrieve-and-rerank pipeline an embedding model's similarity
screens the collection, and a reranker's score reorders the survivors,
superseding that similarity. The two signals differ in kind and in role.
The three scores of \textsc{Gavel} stand in a different relation. With the prior
uniform over the library, each can be read as an estimate of the same
log posterior $\log p(s|x)$, on a scale of its own.

The glance $g$ reads it contrastively, since the span mean that trains
the head minimizes the multi-positive InfoNCE loss of
Section~\ref{sec:glance}, whose optimum is the log ratio of the
posterior to the negative-sampling distribution up to a shift shared by
every skill under a task \citep{oord2018representation}. The vote of
Eq.~(\ref{eq:vote}) re-weights the tokens and holds that scale.
Summing Eq.~(\ref{eq:lchan}) over task positions gives
$\log p_{\mathcal{M}}(x \,|\, r(s))$, the log-likelihood of the task
with the skill as the condition. Under the uniform prior, Bayes' rule
turns this into the log posterior up to another such shift, so the
likelihood $L$ reads it generatively. The division by $|x|$ tempers this
expert by the task's length, so a long task cannot outweigh the other
two. The judgment
$V$ is the model's own posterior on the skill's relevance when asked, a
discriminative reading stated rather than derived.

We therefore rule by the product of experts \citep{hinton2002products},
\begin{equation}
\label{eq:ruling}
S(s|x) \;=\; g(s|x) \;+\; \alpha\, L(s|x) \;+\; \gamma\, V(s|x),
\end{equation}
so that $e^{S}$ multiplies the three experts, and any one of them can
veto a candidate the other two merely tolerate. The coefficients, fitted on validation data, are both
exchange rates that convert the verdict's nats onto the glance's scale and tempering exponents that
discount an overconfident expert. \textsc{Gavel} loads the skill whose
$S$ is highest among those the glance shortlisted.

%% file: sections/4_experiments.tex
\section{Experiments}
\label{sec:exp}

% We evaluate \textsc{Gavel} in settings of increasing realism, from static
% skill libraries (Section~\ref{sec:exp-static}), where we also ablate each
% design choice, through routing mid-rollout from live agent state
% (Section~\ref{sec:exp-traj}), to end-to-end deployment in a real agent
% harness (Section~\ref{sec:exp-harness}).
% Section~\ref{sec:exp-scaling} then studies how routing scales with the
% backbone.

\subsection{Setup}
\label{sec:exp-setup}

Our main experiments run \textsc{Gavel} on Qwen3-32B
\citep{yang2025qwen3} as the frozen agent backbone, where the matrix-entropy criterion of Section~\ref{sec:glance}
places the read-out after block 45 of the model's 64
(Appendix~\ref{app:layer}). The projections $W_q$ and $W_s$, with 7.9M
parameters in total, are trained once at temperature $\tau = 40$
on 51{,}104 of SkillRet's training queries \citep{cho2026skillret},
which Qwen3.5 wrote over 9{,}084 skills collected from
public repositories. The cover radius $\varepsilon = 0.83$ is fixed on the same split, and
shrinks the skill banks about $8.5\times$ at a cost of at most 1.6
points on any benchmark below. The exchange rates $(\alpha, \gamma) = (1.0, 0.025)$ and the pruning margin
$\Delta = 0.133$ are calibrated on SkillRet's validation split. A verdict
goes only to the candidates whose glance score lies within $\Delta$ of the best,
roughly nine on average. With the projections and these four scalars held fixed, every
number on every other library is zero-shot. Appendix~\ref{app:gemma} rebuilds the pipeline on a Gemma
backbone and repeats the main comparison. Appendix~\ref{app:more-ablations}
collects further ablations.

\subsection{Routing from written tasks}
\label{sec:exp-static}

We first evaluate routing from written tasks on three benchmarks. SkillRet's
official test set (v1) pairs 4{,}997 Claude-written queries with a
library of 6{,}660 held-out skills. SRA-Bench \citep{su2026skill}
changes the genre on both sides, with tasks from six reasoning and
coding benchmarks and 26{,}262 skills, most collected from the web. It
ships no test split, so we evaluate a stratified sample of 861 tasks.
Eval-Core, SkillRouter's own benchmark adapted from SkillsBench
\citep{li2026skillsbench}, poses 75 real \texttt{task.md} queries
against 78K documents in two pools \citep{zheng2026skillrouter}.

A router is scored by whether the single skill it commits to serves the
task, Hit@1 over the full library. On all three benchmarks gold
labels miss adequate skills, as exhaustive annotation is impractical.
We therefore score Hit@1 by adjudication. GPT-5.6 Sol (high)
\citep{openai2026gpt56sol} compares the committed skill
with every gold in both orders, given only the task and both
documents. The skill is credited when it wins at least as often as it
loses across all golds. Appendix~\ref{app:benchmarks} gives examples of under-annotated
queries, the unadjudicated scores, and a human check of the verdicts.

\subsubsection{Comparison with current practice}
\label{sec:exp-static-main}

\textsc{Gavel} is compared against the best of current practice.
Progressive disclosure cannot survey thousands of skills in context, so
we give it a retrieval front end. Qwen3-Embedding-8B
\citep{zhang2025qwen3embedding}, the top open embedding model
on MTEB, shortlists twenty skills for the agent to pick from. Retrieve-and-rerank comes in
the form SkillRouter ships, a 0.6B embedder and a 0.6B reranker trained
for skill selection \citep{zheng2026skillrouter}, and in a scaled-up
general-purpose form, Qwen3-Embedding-8B followed by Qwen3-Reranker-8B.
Alongside them we keep BM25 as the lexical anchor, and SKILLRET-Emb-0.6B, trained on SkillRet by its authors
\citep{cho2026skillret}.

As shown in Figure~\ref{fig:main-results}, \textsc{Gavel} leads all
three benchmarks at both stages. The full pipeline beats the
strongest baseline by 3.8 points on SkillRet, 13.4 on SRA-Bench, and
1.3 to 2.7 on Eval-Core. The glance alone is the strongest retrieval
stage on SRA-Bench and Eval-Core while coming within a point of the
best embedder on SkillRet (candidate-pool recall tells much the same
story, Appendix~\ref{app:pool-recall}). The two trained embedders,
SKILLRET-Emb-0.6B and SkillRouter's, match the glance on SkillRet,
whose queries and skills resemble their training data. However, once
the query style shifts from conversational user requests to the
\texttt{task.md} project files of Eval-Core, both begin to lag.
When the skills shift as well, from \texttt{SKILL.md} files written
for agents to the procedural web pages of SRA-Bench, both fall below
even BM25, the out-of-domain failure dense retrievers are known for
\citep{thakur2021beir}.

\begin{figure}[htbp]
\centering
\includegraphics[width=\textwidth]{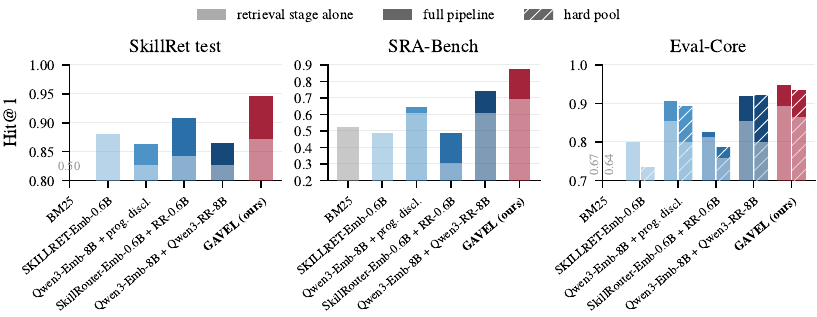}
\caption{Routing from written tasks. Light fill is the retrieval stage
committing its top-1 alone, dark fill the full pipeline built on it, and
on Eval-Core the hatched bar is the hard pool. Where BM25 falls below
the plotted range its score is printed in place of the bar.}
\label{fig:main-results}
\end{figure}

The untrained Qwen3-Embedding-8B exhibits the
opposite preference. On Eval-Core and SRA-Bench, matching a task file
against a skill document is close to the web-scale retrieval it was
trained for, and it fares well. SkillRet's conversational requests
share little vocabulary with the skill body
\citep{shi2025toolret,cho2026skillret}, and there it trails even
the far smaller trained embedders. Whether the shortlist is then
handed to progressive disclosure or to a dedicated reranker, these
widely practiced pipelines place an external model of 1.2B to 16B
parameters beside the agent, against our 7.9M. Yet none of them closes
the gap to a router read out of the agent itself.

\subsubsection{Ablations}
\label{sec:exp-static-ablation}

We next replace one design choice at a time, keep the rest fixed, and rerun the three benchmarks. Every replacement falls below the unmodified router on every benchmark (Figure~\ref{fig:ablations}).

\begin{figure}[htbp]
\centering
\includegraphics[width=0.8\textwidth]{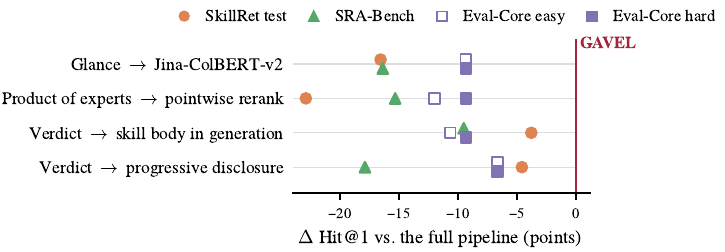}
\caption{Ablations. Each row replaces one design choice of
\textsc{Gavel} and shows how far the full pipeline drops on each
benchmark, in points against the unmodified router (red line). On
Eval-Core the hollow marker is the easy pool and the filled one the
hard pool.}
\label{fig:ablations}
\end{figure}

\textbf{Glance $\to$ Jina-ColBERT-v2.} The glance could win merely
by scoring at a finer granularity, every token pair rather than one
vector pair, and not by what it reads from the agent. So
Jina-ColBERT-v2 \citep{jha2024jinacolbert}, a retriever trained for
exactly this late-interaction scoring, takes the glance's place. It
nominates the shortlist and stands in for $g$ in the ruling, with
$(\alpha, \gamma)$ calibrated exactly
as ours are, yet loses badly on every benchmark even with the verdict
behind it. The glance profits from the far better compression of the
agent LLM it reads, not from how the tokens are compared.

\textbf{Product of experts $\to$ pointwise rerank.} We restore the
classical division of labor, in which retrieval only nominates and a reranker
alone decides. The ruling commits to $V$, the backbone's judgment of each
candidate on its own, and discards $g$ and $L$. This runs into the
poor calibration of self-reported judgment
\citep{qin2024pairwise}, and on SkillRet it
scrambles the ambiguous queries.

\textbf{Verdict $\to$ skill body in generation.}
This variant appends the skill after the task, takes the mean
log-likelihood of its body in place of the query-likelihood of
Eq.~(\ref{eq:lchan}), and reads $V$ with the same question. But a body's likelihood speaks more to how much the
skill resembles the model's own writing than to whether it serves the
task. The query-likelihood carries this bias equally for every skill,
leaving their order intact. Indeed the calibration discards the
document-likelihood outright,
driving $\alpha$ to zero, and on the remaining two signals the
variant falls visibly behind on every benchmark.

\textbf{Verdict $\to$ progressive disclosure.} Handing the glance's
twenty candidates to the menu of progressive disclosure tests whether
the verdict earns its forward passes over letting the agent read the
shortlist itself. The menu comes closest on Eval-Core, whose curated
skills carry descriptions that summarize them well. But it falls far
behind on SRA-Bench, where telling the right skill apart takes the
long procedural detail in its body, yet the pick rests on metadata
that cannot hold it. This is the information loss
\citet{zheng2026skillrouter} identify, which the verdict escapes by
reading the full body. Placing the twenty full bodies in the menu
itself does not close the gap either
(Appendix~\ref{app:fullbody}).

\subsection{Routing mid-rollout}
\label{sec:exp-traj}

\begin{figure}[t]
\centering
\includegraphics[width=\linewidth]{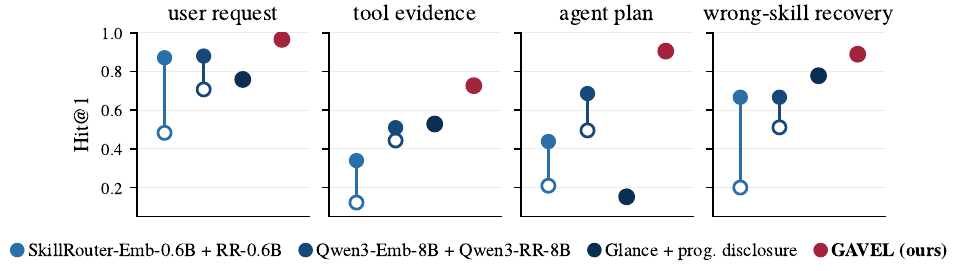}
% SkillRouter values are the st37 old judge round; swap in the
% bl2-protocol re-judge when it lands.
\caption{Routing mid-rollout on SkillTraj, a panel per scenario. A
retrieve-and-rerank pipeline has two points, hollow when it swallows
the full trajectory, filled when it keeps only the last message.}
\label{fig:traj}
\end{figure}

The benchmarks above hand the router a written task that states its need in a single clean turn. A live agent instead works from a long multi-turn context, and the need may surface only as it works through the task.
We therefore build SkillTraj, a new skill-routing benchmark of
simulated agent rollouts. Each of its 372 trajectories is a multi-turn
dialogue with tool calls, rendered in the backbone's chat template, and
marks the point at which a skill from the held-out SkillRet library of
6{,}660 becomes needed. A router is scored on the skill it commits to
at that point.

The points cover four scenarios, by where the need comes
from: stated outright by the user (\emph{user request}, 116), revealed
by a tool result (\emph{tool evidence}, 106), created by the agent's
own written plan (\emph{agent plan}, 105), or arising after the agent
has read the wrong skill and must reroute with the misleading document
still in context (\emph{wrong-skill recovery}, 45). Each trajectory is
written by GPT-5.6 Sol \citep{openai2026gpt56sol} and admitted by a
Claude judge only if the need is absent before the marked point, present after it, and solved by the gold skill (further checks in Appendix~\ref{app:skilltraj}).

On these trajectories \textsc{Gavel} reads the context where it
already sits, inside the agent's own forward pass. The glance sets the
vote decay to $w_i = 2^{-a_i/64}$, where $a_i$ counts the tokens
between position $i$ and the decision point, so recent messages weigh
most. The verdict then continues from the
most recent message, putting the backbone's reasoning to
work on a need only partly spelled out. A conventional
retrieve-and-rerank pipeline must instead decide what to render for
its retriever, either the whole noisy context, whose embedding mixes
every topic the trajectory has touched, or only the latest turn, at the
risk of losing the evidence that matters. Figure~\ref{fig:traj}
compares \textsc{Gavel} against each pipeline taken both ways and
against progressive disclosure, constrained to decode a menu index.

\textsc{Gavel} leads every scenario, by 8.6 to 21.9 points over the
strongest other system. These trajectories sit far from anything
SkillRouter was trained on, and it falls behind the larger untrained Qwen3 models in every
scenario, furthest when the full multi-turn noise has to be
swallowed. Yet the larger models cannot iron the noise out either; in
all four scenarios they are better served by the clean last message
than by the full context that holds all the usable evidence.
Progressive disclosure brings back the coarse, metadata-bound picking
that hurt it most on SRA-Bench. It holds its own on \emph{wrong-skill
recovery}, where the agent has just watched a skill fail and sums up
the missing capability in a sentence that metadata alone can match,
but collapses on \emph{agent plan}, where the freshly
drafted plan carries details only the skill body can confirm.

\subsection{End-to-end deployment}
\label{sec:exp-harness}

We further integrate \textsc{Gavel} end to end into
mini-swe-agent \citep{yang2024sweagent}, a minimal bash-agent harness. While every benchmark above hands the router its moment, a
live rollout requires a gate that decides when to attempt a skill
call. We therefore train one, taking the decision points of SkillTraj
as positives and trajectories built the same way but never warranting
a skill as negatives. It reads the glance scores the
preceding tokens have accumulated over the library and the final-layer
state from which the model predicts its next token. Once it stays on
for two consecutive tokens, \textsc{Gavel} attempts a selection.
If the verdict judges even the winner unfit ($V < 0$), it loads nothing.

\begin{table}[htbp]
\centering
\caption{Trigger rate on Skill-Use, the fraction of tasks on which
the correct skill is loaded, with frontier numbers from its paper.}
\label{tab:skilluse}
\small
\begin{tabular}{lc}
\toprule
System & Trigger \\
\midrule
\multicolumn{2}{l}{\emph{Progressive disclosure}} \\
GLM-5.1 (in Codex) & .706 \\
MiniMax-M3 (in Codex) & .864 \\
Qwen3.6-Max (in Codex) & .684 \\
DeepSeek-V4-Pro (in Codex) & .650 \\
Qwen3-32B (in mini-swe-agent) & .011 \\
\midrule
\multicolumn{2}{l}{\emph{Retrieve and rerank}} \\
Qwen3-Emb-8B + Qwen3-Reranker-8B & .897 \\
SkillRouter-Emb-0.6B + Reranker-0.6B & .800 \\
\midrule
\multicolumn{2}{l}{\emph{Ours}} \\
Qwen3-32B + \textsc{Gavel} (in mini-swe-agent) & \textbf{.909} \\
\bottomrule
\end{tabular}
\end{table}
We test the integration on Skill-Use \citep{han2026skilluse}, a
benchmark of 177 executable tasks over a curated library of 79 skills,
built to measure how reliably a model and harness invoke the right
skill under progressive disclosure. As
Table~\ref{tab:skilluse} shows, under the progressive-disclosure
prompt (Appendix~\ref{app:harness}) Qwen3-32B almost
never reads a skill at all. With
\textsc{Gavel} integrated, however, the same frozen model triggers
the correct skill in more than 90\% of the tasks, drawing on ability
it already carried (in another 5\% the router holds its fire and
loads nothing). That puts it ahead of every frontier model
in Codex, and ahead of the retrieve-and-rerank pipelines, which load
a skill on every trajectory whether one is needed or not.
Appendix~\ref{app:harness} walks through transcripts in which
\textsc{Gavel} invokes a skill mid-rollout, and Appendix~\ref{app:cost}
compares its cost with progressive disclosure over a whole session.

\subsection{Scaling with backbone capability}
\label{sec:exp-scaling}

% Keep the explanation and its compact plot together, with aligned tops.
\noindent
\begin{minipage}[t]{0.49\linewidth}
\vspace{0pt}
Much of the appeal of progressive disclosure is that its judgment of
skill selection improves for free as the agent LLM does.
\textsc{Gavel} in fact shares this advantage, as
Figure~\ref{fig:scaling} shows. On SRA-Bench both the glance stage
and the full form climb as the backbone grows in size and steps up in
generation. On a 0.6B model, one that nearly always
fails at picking from a metadata menu, \textsc{Gavel} already beats the 32B under
progressive disclosure, which here has the same backbone pick from a
menu of twenty candidates retrieved by Qwen3-Emb-8B (adjudicated
recall $\approx .93$). Even on Qwen3.8-27B, whose agentic ability
draws level with far larger frontier models on many benchmarks
\citep{qwen2026qwen38}, picking from the metadata
menu still trails \textsc{Gavel} on the same backbone by 12.2 points.
\end{minipage}\hfill
\begin{minipage}[t]{0.47\linewidth}
\vspace{0pt}
\captionsetup{type=figure}
\centering
\includegraphics[width=\linewidth]{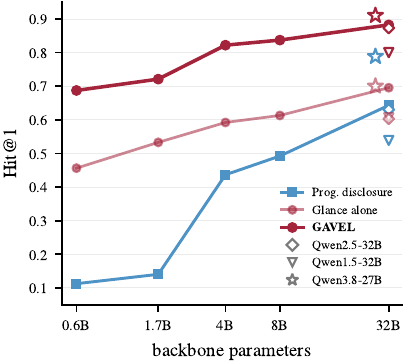}
\caption{Routing accuracy on SRA-Bench as the backbone grows more
capable, across Qwen3 sizes (solid lines), across generations at a
fixed 32B (hollow markers), and onto
Qwen3.8-27B (stars).}
\label{fig:scaling}
\end{minipage}
\par

%% file: sections/5_conclusion.tex
\section{Conclusion}
\label{sec:conclusion}

We presented \textsc{Gavel}, a router that returns skill selection to
the agent LLM without placing anything in its context. The glance
scores the full library from a mid-layer read-out of the frozen
backbone, the verdict re-scores the shortlist from the model's own
generative and discriminative predictions, and two linear projections
are all that is trained. Across three public benchmarks and SkillTraj
this native read-out outperforms progressive disclosure and
retrieve-and-rerank pipelines by up to 21.9 points, and in a live harness it has Qwen3-32B load the correct skill more
often than frontier models of up to 1.6T parameters. The frozen backbone already carries the capability that skill
selection requires, and a learned read-out is enough to put it to
work. Whether the same read-out also routes tools, memories, or
MCP servers from an agent's own forward passes is a question we leave
open.

%% file: appendix/01_native_readouts.tex
\section{Native read-outs that fail}
\label{app:native}

The glance reads $\mathcal{M}$ through two trained maps and aggregates
the per-token scores under a hand-crafted decay $w_i$, rather than
through machinery the model already owns. This appendix records three
alternatives that lean further toward the native machinery. Each fails,
and each failure fixes one piece of the adopted design.

\paragraph{Native attention keys as the matching space.}
Attention already computes token-level query--key matches, so the most
native glance would use them directly, indexing a skill by the
pre-rotary keys the backbone's own heads assign to its tokens and
matching the task's native queries against that bank. Without training,
the best of 24 configurations of head selection, low-rank whitening,
and aggregation reaches Hit@1 $.001$ on the SkillRet validation library
of 10{,}123 skills, against $.918$ for the trained glance head.
On SRA-Bench it scores exact zeros, and even its recall at 20 falls below
random ranking. Most of the loss comes from calibration, since native scores are not
comparable across documents encoded in separate sequences. Restricted
to a 20-candidate shortlist, the same read-out scores $.099$, twice
in-pool chance. Training a mixer over per-head scores does not repair
it. The mixture becomes a usable shortlist reranker but drops 7.6
points when asked to rank the full library, and its weights are
library-bound, losing 1.2 points on the library they were fitted on
and 24 on another. Equal-weight and non-negative mixtures collapse
further, which rules out negative weights as the cause. The native
keys do carry a real signal, one to two orders of magnitude too
faint. The best single head reaches $.084$ alone, and the ranking of
heads is stable across libraries, yet no a priori criterion finds the
good ones (the head placing the most attention mass on the gold
document ranks 29th by retrieval). Trained projections are what turn
this signal into a usable one.

\paragraph{Initializing from native heads.}
The trained head still needs an initialization and a shape, and the
native machinery might supply both. For the initialization, we stack
the query and key maps of six strong native heads up to the glance's
dimension and start $W_q$ and $W_s$ from them. For the shape, we keep
six separate heads and average their scores. On the SkillRet validation library, native
initialization ends 0.7 points below random initialization, the
multi-head shape 0.6 points below the single wide head, and the two
combined lose 4.6. Selecting which heads to stack contributes nothing
either, as six randomly chosen heads beat the six the mixer above
weights highest by 1.3 points. The glance head is therefore a single
map, randomly initialized.

\paragraph{Trained retrieval tokens.}
The decay $w_i$ in the vote of Section~\ref{sec:glance} is the one
hand-crafted piece of the glance, and native attention is the natural
candidate to replace it. We append 64 trained retrieval tokens to the
context, let each aggregate the preceding tokens through the
backbone's own attention, pass their hidden states through $W_q$, and
average their votes. Trained on the same corpus as the glance, the
tokens learn a recency read-out instead, settling their attention on
the most recent stretch of context rather than allocating it by
content. We attribute this more to the data than to the idea. The
plentiful supervision consists of clean single-turn requests, where
the recent context is the task and recency is a sufficient policy, so
nothing forces the tokens to allocate attention selectively. Yet the
router must survive contexts whose trigger sits buried mid-history.
We therefore keep the hand-crafted decay, and leave learned
aggregation, trained on data that rewards selectivity, to future
work.

%% file: appendix/02_layer_selection.tex
\section{Layer selection}
\label{app:layer}

\begin{figure}[htbp]
\centering
\includegraphics[width=0.5\linewidth]{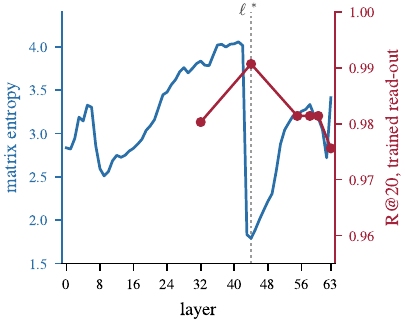}
\caption{Selecting $\ell^\ast$. Left axis: matrix-based entropy of the
token states of a rendered skill, median over 500 skills, sink token
removed. Right axis: R@20 of the read-out trained at each of six
candidate depths, on the validation queries. The trained scan
peaks at the entropy floor the unsupervised criterion picks.}
\label{fig:layer}
\end{figure}

We locate $\ell^\ast$ with a criterion measured on our own backbone and
corpus. For one rendered skill, let $Z$ stack its layer-$\ell$ token
states. The matrix-based entropy
$S_1 = -\sum_k p_k \log p_k$ over the normalized spectrum
$p = \lambda(ZZ^\top) / \operatorname{tr}(ZZ^\top)$ measures how many
directions the states occupy \citep{skean2025layers}. A
massive-activation token dominates this spectrum wherever it appears,
turning the entropy into a measure of the sink rather than the
representation \citep{queipo2026sinks}. Since the spans the glance
reads never contain that token, we drop the top-norm token before
taking the spectrum. Figure~\ref{fig:layer} plots the median over 500
skills in the render $r(s)$. The entropy climbs through the
first two thirds of the depth, collapses from $4.01$ to $1.83$ across
a single block, bottoms out at layer 44 of 64, and recovers toward the final layer.
Layer 44 is the output of block 45, roughly 70\% of the way through the
model. We set $\ell^\ast = 44$.

A trained scan then checks the choice against retrieval. At each of
six candidate depths (layers 32, 44, 55, 58, 60, 63) we train $W_q$
and $W_s$ under one frozen recipe and rank the full library for the
validation queries. Layer 44 comes first on R@20
(Figure~\ref{fig:layer}).

Run zero-shot, the same comparison inverts. By raw cosine over
span-pooled states of the same render on a 4{,}000-document library,
layer 44 is the second-worst of the 34 depths we banked (R@20 $.503$,
against $.916$ at layer 60), and the spread across depths is 44
points. Embedding read-outs peak in
the compression phase \citep{skean2025layers}. Here they do so only
through trained projections, as raw similarity cannot read the compressed states.

%% file: appendix/03_training.tex
\section{Training details}
\label{app:training}

The two projections $W_q$ and $W_s$ are $768 \times 5120$ matrices
without bias, 7.9M parameters in all, initialized at random. Both are
trained on the SkillRet training split with the loss of
Section~\ref{sec:glance} at a fixed temperature $\tau = 40$, in fp32,
with AdamW at a constant learning rate of $10^{-3}$ and weight decay
$0.01$ for 24{,}000 steps. Gradients stop at the backbone's states,
which are computed once and stored. Each step packs up to 256 tasks,
stopping once their gold and negative skills fill a budget of 200K
tokens, and resamples the negatives from the rest of the library. During training a skill contributes at most 8{,}192 tokens of
keys; at evaluation the only truncation is the 30{,}720-token serving
budget (Appendix~\ref{app:prompts}). After the banks are compressed, $W_q$
alone is fine-tuned against the compressed keys for a further 3{,}000
steps, with $W_s$ and the banks frozen. It uses the same loss and
optimizer at a learning rate of $10^{-4}$, chosen on the validation
split. On the written-task benchmarks the vote weights $w_i$ of
Eq.~(\ref{eq:vote}) are uniform, and the decay of
Section~\ref{sec:exp-traj} applies only mid-rollout.

%% file: appendix/04_compression.tex
\section{Compression guarantees}
\label{app:compress}

Throughout, $D = \{d_1,\dots,d_n\}$ is the key set of one skill, every
$\|d_j\| = 1$, and distances are Euclidean. A subset $C \subseteq D$ is an
$\varepsilon$-cover of $D$ if every $d \in D$ has some $c \in C$ with
$\|d - c\| \le \varepsilon$, and $\varepsilon$-separated if
$\|c - c'\| \ge \varepsilon$ for all distinct $c, c' \in C$. The covering
number $\mathcal{N}_\varepsilon(D)$ is the smallest size of an
$\varepsilon$-cover of $D$ by its own points, and the packing number
$\mathcal{P}_\varepsilon(D)$ the largest size of an
$\varepsilon$-separated subset.

\begin{proof}[Proof of Proposition~\ref{prop:distortion}]
The upper bound holds because $C_s \subseteq D_s$, so the max on the left
ranges over a subset. For the lower bound, let $d^\ast$ attain
$\max_{d \in D_s} \langle q, d\rangle$ and let $c \in C_s$ cover it,
$\|d^\ast - c\| \le \varepsilon$. By Cauchy--Schwarz,
\begin{equation*}
\langle q, c\rangle
= \langle q, d^\ast\rangle - \langle q,\, d^\ast - c\rangle
\ge \langle q, d^\ast\rangle - \|q\|\,\|d^\ast - c\|
\ge \max_{d \in D_s}\,\langle q, d\rangle - \varepsilon . \qedhere
\end{equation*}
\end{proof}

The bound carries through the vote of Eq.~(\ref{eq:vote}). Fix a
task $x$ with weights $w_i$ and $W = \sum_i w_i$, and write
$\tilde m_i(s)$, $\tilde{\mathcal{T}}_i$ and $\tilde g(s|x)$ for the
score, top-$k$ set and vote computed over $\varepsilon$-covers of every
bank. For a skill $s$ let $a_i(s)$ be the $k$-th largest of
$\{m_i(s') : s' \ne s\}$, and let
$B_x(s) = \{\, i : |m_i(s) - a_i(s)| \le \varepsilon \,\}$ be the tokens
on which $s$ sits within $\varepsilon$ of the top-$k$ boundary.

\begin{corollary}[Vote stability]
\label{cor:vote}
For every task $x$, skill $s$ and $\varepsilon \le 1$,
\begin{equation*}
\bigl|\tilde g(s|x) - g(s|x)\bigr|
\;\le\; \varepsilon + \frac{1}{W}\sum_{i \in B_x(s)} w_i .
\end{equation*}
In particular, when $B_x(s)$ is empty, so that $s$ clears every
token's top-$k$ boundary by more than $\varepsilon$, membership is
unchanged and $g(s|x) - \varepsilon \le \tilde g(s|x) \le g(s|x)$, so
compression lowers the vote by at most $\varepsilon$ and never raises it.
\end{corollary}

\begin{proof}
Proposition~\ref{prop:distortion} applied to every skill gives
$m_i(s') - \varepsilon \le \tilde m_i(s') \le m_i(s')$ for all $i$ and
$s'$. Take $i \notin B_x(s)$. If $m_i(s) > a_i(s) + \varepsilon$, then
$s \in \mathcal{T}_i$, and
$\tilde m_i(s) \ge m_i(s) - \varepsilon > a_i(s)$ still exceeds the
$k$-th largest compressed score among the other skills, which can only
have fallen, so $s \in \tilde{\mathcal{T}}_i$. If instead
$m_i(s) < a_i(s) - \varepsilon$, then $s \notin \mathcal{T}_i$, and
the $k$ other skills scoring at least $a_i(s)$ still score at least
$a_i(s) - \varepsilon > m_i(s) \ge \tilde m_i(s)$ after compression,
so $s \notin \tilde{\mathcal{T}}_i$. Membership is therefore unchanged
off $B_x(s)$, and there the term $w_i\, m_i(s)$, when present, falls
by at most $\varepsilon w_i$ and never rises, since
$\tilde m_i(s) \le m_i(s)$. On $B_x(s)$ a term may appear or vanish,
moving by at most $w_i$ since every score is an inner product of unit
vectors, or change by at most $\varepsilon w_i \le w_i$. Summing over
$i$ and dividing by $W$ gives the bound.
\end{proof}

The bound is a worst case. Cauchy--Schwarz is tight in
Proposition~\ref{prop:distortion} only when $q$ is collinear with
$d^\ast - c$, which for unit vectors forces
$\langle q, d^\ast\rangle \le \varepsilon/2$, a key that barely matches
the query to begin with. On the tokens that decide a routing, $q$ lies
close to its best key $d^\ast$, and there
$\langle q, d^\ast - c\rangle \approx 1 - \langle d^\ast, c\rangle
= \|d^\ast - c\|^2/2 \le \varepsilon^2/2$. At the
$\varepsilon = 0.83$ of Section~\ref{sec:exp-setup} this is $0.34$,
against $0.83$ for the linear bound, in line with the cost of at most 1.6 points reported there.

\begin{proposition}[Bank size]
\label{prop:size}
The traversal returns a $C_s$ that is simultaneously an
$\varepsilon$-cover and $\varepsilon$-separated (retained keys are
pairwise at least $\varepsilon$ apart), hence
$\mathcal{N}_\varepsilon(D_s) \le |C_s| \le \mathcal{P}_\varepsilon(D_s)$,
where $\mathcal{N}_\varepsilon$ and $\mathcal{P}_\varepsilon$ are the
covering and packing numbers of $D_s$ at scale $\varepsilon$.
\end{proposition}

\begin{proof}[Proof of Proposition~\ref{prop:size}]
The traversal stops precisely when every key of $D_s$ lies within
$\varepsilon$ of the retained set, so $C_s$ is an $\varepsilon$-cover.
For separation, consider the moment a key was retained. It was the
farthest key from the set retained so far, and the traversal had not yet
stopped, so its distance to every earlier key was at least $\varepsilon$.
Applying this to the later of any two retained keys shows $C_s$ is
$\varepsilon$-separated. The size bounds then follow from the
definitions: an $\varepsilon$-cover has at least
$\mathcal{N}_\varepsilon(D_s)$ points, and an $\varepsilon$-separated
subset at most $\mathcal{P}_\varepsilon(D_s)$.
\end{proof}

The two ends of the bracket are within one scale factor of each other,
$\mathcal{P}_{2\varepsilon}(D) \le \mathcal{N}_\varepsilon(D) \le
\mathcal{P}_\varepsilon(D)$. The upper inequality holds because a
maximal $\varepsilon$-separated set is already an $\varepsilon$-cover,
as any key farther than $\varepsilon$ from all of it could otherwise be
added. The lower one holds because a ball of radius $\varepsilon$
contains at most one point of a $2\varepsilon$-separated set.

%% file: appendix/05_gemma.tex
\section{A non-Qwen backbone}
\label{app:gemma}

Every number in the body of the paper reads a Qwen backbone. This
appendix rebuilds the pipeline end to end on
\texttt{gemma-4-31b-it} \citep{gemmateam2026gemma4}, a backbone of
comparable size from a different family, with a different chat
template, interleaved sliding-window attention, and logit
soft-capping. Nothing is transferred from the Qwen build. The read-out
depth is selected anew on SkillRet validation and lands at block 34 of
the model's 60, the projections are retrained with the recipe of
Section~\ref{sec:exp-setup}, the cover radius is refit on the training
split at $\varepsilon = 0.77$, and the exchange rates
$(\alpha, \gamma) = (0.15, 0.003)$ are recalibrated on the validation
split. The pair sits well below the Qwen values of
Section~\ref{sec:exp-setup} because the exchange rates absorb the scale
of the signals they convert. Among the twenty candidates a validation
query shortlists, both verdict read-outs spread about three times wider
on Gemma than on Qwen3-32B (median within-query standard deviation 14.8
against 4.5 nats for $V$, 0.75 against 0.23 nats per token for $L$),
while the glance spreads about the same (0.061 against 0.052), so a
coefficient a third the size carries the same weight. The three
benchmarks are then
scored as in Section~\ref{sec:exp-static}, by the same judge under the
same rule.

\begin{figure}[htbp]
\centering
\includegraphics[width=\textwidth]{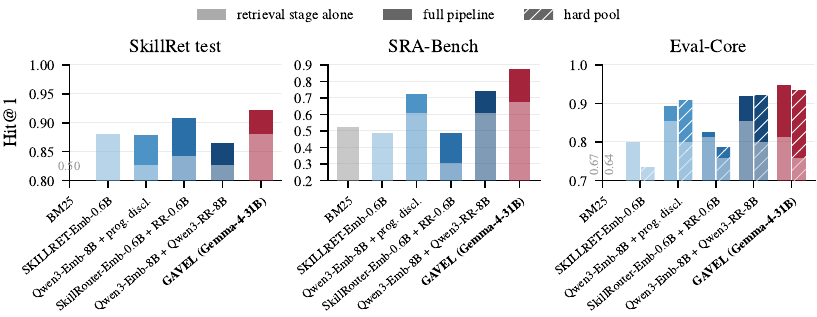}
\caption{The comparison of Figure~\ref{fig:main-results} with the
routing backbone changed to Gemma-4-31B. The baselines do not involve
that backbone, so their bars carry over from Figure~\ref{fig:main-results};
the rebuilt parts are \textsc{Gavel} itself and the in-context picker
of progressive disclosure. Bars read as before, light fill for the
retrieval stage alone, dark fill for the full pipeline, hatched for
Eval-Core's hard pool, and a below-range BM25 score printed in place
of its bar.}
\label{fig:gemma}
\end{figure}

Figure~\ref{fig:gemma} shows the result. The shape survives the family
change, and \textsc{Gavel} again posts the highest adjudicated score
on all three benchmarks. Compressing the skill banks changes its
SkillRet and SRA-Bench scores by at most 0.7 points.

%% file: appendix/06_ablations.tex
\section{Additional ablations}
\label{app:more-ablations}

\subsection{Progressive disclosure over full skill bodies}
\label{app:fullbody}

The progressive-disclosure baseline of
Section~\ref{sec:exp-static-main} shows the agent twenty names and
descriptions. Here we rerun it with the menu expanded. The shortlist
and the picker are unchanged, but each entry now carries the full
\texttt{SKILL.md} below its metadata, so the agent reads twenty
complete skills before answering with an index. With Qwen3-32B's
window extended to 131K tokens by YaRN, the twenty bodies fit for
every query of every benchmark; the longest prompt is 124K
tokens. Table~\ref{tab:fullbody} gives the scores.

\begin{table}[ht]
\centering
\small
\setlength{\tabcolsep}{3.5pt}
\caption{Progressive disclosure picking from the twenty full skill
bodies, against the metadata menu of Figure~\ref{fig:main-results}
and \textsc{Gavel}, under the same retrieval front end, picker, and
judge.}
\label{tab:fullbody}
\begin{tabular}{l cc cc cc cc}
\toprule
& \multicolumn{2}{c}{SkillRet test} & \multicolumn{2}{c}{SRA-Bench}
& \multicolumn{2}{c}{Eval-Core easy} & \multicolumn{2}{c}{Eval-Core hard}\\
\cmidrule(lr){2-3}\cmidrule(lr){4-5}\cmidrule(lr){6-7}\cmidrule(lr){8-9}
& raw & adj. & raw & adj. & raw & adj. & raw & adj.\\
\midrule
Metadata menu & .6526 & .8635 & .5633 & .6434 & .5467 & .9067 & .5467 & .8933\\
Full-body menu & .6510 & .8417 & .6353 & .7131 & .5867 & .9067 & .5333 & .8667\\
\textbf{\textsc{Gavel}} & .8459 & .9460 & .8223 & .8734 & .6933 & .9467 & .6933 & .9333\\
\bottomrule
\end{tabular}
\end{table}

The bodies help only on SRA-Bench, whose entries are procedural web
pages with uninformative titles; there both scores rise by seven
points. On SkillRet the adjudicated score falls by 2.2 points, and on
Eval-Core the changes amount to one to three queries. \textsc{Gavel}
stays 16 adjudicated points ahead on SRA-Bench and 10 on
SkillRet. Missing information does not explain the remaining errors:
on SkillRet the gold's complete body is in the prompt for 92\% of the
queries, and the picker chooses another skill in 29\% of them. The
context cost, meanwhile, is large. The median prompt grows from 1.3K
tokens to 28K--52K, and a single pick takes 26 to 124 seconds of
prefill.

\subsection{Dropping one verdict signal}
\label{app:signal-ablation}

The ablation of Figure~\ref{fig:ablations} that commits to $V$ alone
removes the glance and the likelihood together. Here we remove one
verdict signal at a time from the ruling of Eq.~(\ref{eq:ruling}),
leaving the glance, its shortlist, and the compressed banks as they
are, so that $g + \alpha L$ rules on the likelihood alone and
$g + \gamma V$ on the judgment alone.
Table~\ref{tab:signal-ablation} reports both beside
the glance alone and the full ruling, with SkillTraj pooled over its
372 trajectories.

\begin{table}[ht]
\centering
\small
\setlength{\tabcolsep}{3pt}
\caption{Dropping one verdict signal from the ruling, with the glance
alone and the full ruling for reference.}
\label{tab:signal-ablation}
\begin{tabular}{l cc cc cc cc cc}
\toprule
& \multicolumn{2}{c}{SkillRet test} & \multicolumn{2}{c}{SRA-Bench}
& \multicolumn{2}{c}{Eval-Core easy} & \multicolumn{2}{c}{Eval-Core hard}
& \multicolumn{2}{c}{SkillTraj}\\
\cmidrule(lr){2-3}\cmidrule(lr){4-5}\cmidrule(lr){6-7}\cmidrule(lr){8-9}\cmidrule(lr){10-11}
& raw & adj. & raw & adj. & raw & adj. & raw & adj. & raw & adj.\\
\midrule
Glance alone ($g$) & .7681 & .8725 & .6272 & .6957 & .6800 & .8933 & .6400 & .8667 & .7043 & .7366\\
$g + \alpha L$ & .8241 & .9266 & .7758 & .8153 & .7333 & .8400 & .7200 & .8133 & .7769 & .8172\\
$g + \gamma V$ & .6870 & .8981 & .7561 & .8351 & .6667 & .9067 & .6533 & .8933 & .7769 & .8253\\
\textbf{\textsc{Gavel}} ($g + \alpha L + \gamma V$) & .8459 & .9460 & .8223 & .8734 & .6933 & .9467 & .6933 & .9333 & .8253 & .8710\\
\bottomrule
\end{tabular}
\end{table}

The full ruling scores highest on every benchmark, ahead of the
better single-signal variant by 1.9 adjudicated points on SkillRet,
3.8 on SRA-Bench, 4.6 on SkillTraj, and three queries on each
Eval-Core pool. The two signals do not lift the glance evenly. The
judgment raises it on every library, whereas the likelihood raises it
by 5.4 points on SkillRet, 12.0 on SRA-Bench, and 8.1 on SkillTraj
but lowers it by four queries on each Eval-Core pool. Which single
signal is worth more depends on the library. The likelihood carries
SkillRet, whose queries spell out the task in enough detail for
the right skill to anticipate them closely. The judgment instead carries SRA-Bench and SkillTraj, where the
question is whether a procedural page or a mid-rollout state calls for
the skill, which the model answers better when asked than when made to
predict the task's wording. The two signals miss different queries,
and the product of experts keeps the union of what each gets right.

\subsection{Replacing the glance with an off-the-shelf embedder}
\label{app:emb-ablation}

The Jina-ColBERT-v2 variant of Figure~\ref{fig:ablations} replaces
the glance with a late-interaction retriever and loses on every
benchmark. Here the replacement is Qwen3-Embedding-8B, the largest
retrieval stage of Figure~\ref{fig:main-results}. It nominates the
twenty candidates and its cosine similarity stands in for $g$ in the
ruling of Eq.~(\ref{eq:ruling}), with $(\alpha, \gamma)$ calibrated
on the validation split exactly as ours are, while the verdict passes
are unchanged. Table~\ref{tab:emb-ablation} reports the variant
beside the embedder alone, the embedder followed by Qwen3-Reranker-8B
as in Figure~\ref{fig:main-results}, the glance alone, and the full
ruling, with SkillTraj pooled over its 372 trajectories and the
embedder given the last message, as in Figure~\ref{fig:traj}.

\begin{table}[ht]
\centering
\small
\setlength{\tabcolsep}{3pt}
\caption{Replacing the glance with Qwen3-Embedding-8B, with the
embedder alone, the embedder under Qwen3-Reranker-8B, the glance
alone, and the full ruling for reference.}
\label{tab:emb-ablation}
\begin{tabular}{l cc cc cc cc cc}
\toprule
& \multicolumn{2}{c}{SkillRet test} & \multicolumn{2}{c}{SRA-Bench}
& \multicolumn{2}{c}{Eval-Core easy} & \multicolumn{2}{c}{Eval-Core hard}
& \multicolumn{2}{c}{SkillTraj}\\
\cmidrule(lr){2-3}\cmidrule(lr){4-5}\cmidrule(lr){6-7}\cmidrule(lr){8-9}\cmidrule(lr){10-11}
& raw & adj. & raw & adj. & raw & adj. & raw & adj. & raw & adj.\\
\midrule
Qwen3-Emb-8B alone & .6540 & .8273 & .5528 & .6074 & .5600 & .8533 & .5200 & .8000 & .5511 & .5806\\
\quad + Qwen3-Reranker-8B & .6918 & .8645 & .6458 & .7398 & .6667 & .9200 & .6533 & .9200 & .6774 & .6935\\
\quad + verdict & .7232 & .9168 & .7956 & .8513 & .6400 & .9333 & .6400 & .9200 & .7339 & .7796\\
\addlinespace
Glance alone ($g$) & .7681 & .8725 & .6272 & .6957 & .6800 & .8933 & .6400 & .8667 & .7043 & .7366\\
\textbf{\textsc{Gavel}} & .8459 & .9460 & .8223 & .8734 & .6933 & .9467 & .6933 & .9333 & .8253 & .8710\\
\bottomrule
\end{tabular}
\end{table}

The verdict lifts the embedder's shortlist by 9.0 adjudicated points
on SkillRet, 24.4 on SRA-Bench, and 19.9 on SkillTraj. On that
same shortlist it outscores Qwen3-Reranker-8B on the three larger
libraries, by 5.2, 11.2, and 8.6 points, and matches it on Eval-Core.
The variant
nevertheless stays behind \textsc{Gavel} on every benchmark, by 2.9
adjudicated points on SkillRet, 2.2 on SRA-Bench, one query on each
Eval-Core pool, and 9.1 on SkillTraj, where the embedder has to be
handed the last message while the glance reads the rollout as it
goes.

The glance's edge over the embedder goes beyond this margin, since
the two draw their signal from different places. An embedder is a
second model that has to run its own forward pass over every task
before it can rank anything, whereas the glance reads the states the
agent produces as it decodes, so the shortlist costs one sweep of the
head over the bank and nothing more. The embedder also stays what it
was trained to be, while the glance is read out of the agent and
climbs with its size (Section~\ref{sec:exp-scaling}). As the agent grows more
capable, the glance alone may come to settle the routing, ranking the
whole library in passing while the model works on the task, with no
verdict behind it.

%% file: appendix/07_prompts.tex
\section{Prompt templates}
\label{app:prompts}

This appendix collects the prompts the experiments run: the three
templates \textsc{Gavel} feeds the backbone, the two baseline forms
that place skill text in the context instead, and the prompt of the
adjudicated Hit@1. Placeholders stand in braces. Every template is
wrapped in the backbone's chat template in non-thinking mode, whose
control tokens are shown where their exact position matters.

\begin{tcolorbox}[colback=red!4!white, colframe=red!55!black,
  fonttitle=\bfseries, title={The skill render $r(s)$}]
\small
\texttt{<|im\_start|>user}\\
Here is an agent skill (SKILL.md):\\
name: \{name\}\\
description: \{description\}

\{body\}

Here is a user task:
\end{tcolorbox}

At installation the glance runs one forward pass over this render and
stores a key for every token of the metadata header and body. The
body enters in full; the only truncation is the 30{,}720-token
serving budget.

\begin{tcolorbox}[colback=red!4!white, colframe=red!55!black,
  fonttitle=\bfseries, title={The task render}]
\small
\texttt{<|im\_start|>user}\\
\{task\}\texttt{<|im\_end|>}\\
\texttt{<|im\_start|>assistant}\\
\texttt{<think>}

\texttt{</think>}
\end{tcolorbox}

On the written-task benchmarks the glance reads the task tokens'
states from this render. Mid-rollout there is no template of ours:
the glance reads the live context exactly as the harness renders it.

\begin{tcolorbox}[colback=red!4!white, colframe=red!55!black,
  fonttitle=\bfseries, title={The verdict continuation}]
\small
\texttt{<|im\_start|>user}\\
Here is an agent skill (SKILL.md):\\
name: \{name\}\\
description: \{description\}

\{body\}

Here is a user task: \{task\}

Does this skill provide what that task needs? Answer yes or
no:\texttt{<|im\_end|>}\\
\texttt{<|im\_start|>assistant}\\
\texttt{<think>}

\texttt{</think>}
\end{tcolorbox}

The continuation extends $r(s)$ token for token. $L$ is the mean
log-likelihood over the task span, and $V$ is read at the final
position with spelling sets
$\mathcal{Y} = \{\texttt{yes}, \texttt{Yes}, \texttt{YES}\}$ and
$\mathcal{N} = \{\texttt{no}, \texttt{No}, \texttt{NO}\}$, each
spelling taken with and without a leading space, all single tokens in
the backbone's vocabulary.

\begin{tcolorbox}[colback=blue!5!white, colframe=blue!75!black,
  fonttitle=\bfseries, title={Progressive disclosure: the metadata
  menu}]
\small
\{task\}

Here is the list of agent skills available to you. Each line is one
skill, written as \texttt{`index. name: description`}.

1. \{name\}: \{description\}\\
2. \{name\}: \{description\}\\
$\cdots$

Which one of the skills above should be loaded to do this task?
Reply with its index number and nothing else.
\end{tcolorbox}

Descriptions enter untruncated. On the written-task benchmarks the
arm decodes greedily and commits to the first in-range integer in the
output; a reply that names no candidate counts as a miss. On
SkillTraj the trajectory context stands in for the task. There the labels
widen to zero-padded \texttt{01}--\texttt{20} so that none is a
prefix of another, and the prompt closes with \texttt{Answer: }.
Decoding is constrained to the label tokens, so the arm always
commits to a candidate.

\begin{tcolorbox}[colback=blue!5!white, colframe=blue!75!black,
  fonttitle=\bfseries, title={Ablation: the skill body in the
  generation}]
\small
\{task\}

Here is an agent skill (SKILL.md):\\
name: \{name\}\\
description: \{description\}

\{body\}

Does this skill provide what that task needs? Answer yes or no:
\end{tcolorbox}

The variant of Section~\ref{sec:exp-static-ablation} reads from this
reversed order the mean log-likelihood of the body span and the same
yes/no log-odds at the final position.

\begin{tcolorbox}[breakable, colback=black!3!white, colframe=black!65,
  fonttitle=\bfseries, title={The adjudication prompt}]
\small
You are judging a skill-routing decision for an AI coding agent.

Below is a task, followed by two candidate skill documents, A and B.
Exactly one of them will be loaded FIRST to help the agent carry out
this task.

Decide which skill is the better first choice for THIS task.

Rules:
\begin{itemize}
\item Judge by whether the skill provides the specific capability the
task actually needs -{}- not by topical similarity, writing quality,
or document length.
\item Honor hard requirements stated in the task (required tools or
libraries, data access, output formats, explicit instructions). A
skill that cannot satisfy them cannot win.
\item A skill that is merely about the task's topic but cannot
perform the needed work loses to one that can.
\item The task may need several skills overall; judge only which of
these two is the more useful FIRST pick.
\item Answer ``tie'' if both serve the task about equally well -{}-
including the case where the two documents are near-identical copies
of the same skill.
\end{itemize}

[TASK]\\
\{task\}

[SKILL A]\\
name: \{name\_a\}\\
description: \{desc\_a\}

\{body\_a\}

[SKILL B]\\
name: \{name\_b\}\\
description: \{desc\_b\}

\{body\_b\}

Answer with a single JSON object and nothing else:\\
\texttt{\{"winner": "A" | "B" | "tie", "reason": "<one sentence>"\}}
\end{tcolorbox}

The judge fills A with the committed skill and B with a gold, then
once more with the roles swapped, and is never told which side is
which.

%% file: appendix/08_benchmarks.tex
\section{Benchmark details and disclosures}
\label{app:benchmarks}

Section~\ref{sec:exp-static} scores every router by adjudicated
Hit@1. This appendix shows the under-annotation that calls for it,
pairs every adjudicated score with its raw one, checks the judge
against human labels, reports the candidate-pool recall of every
retrieval stage, and details the construction of SkillTraj. Eval-Core
is scored on both pools its authors ship, the easy pool of 78{,}361
documents and the hard pool, which adds 780 adversarial distractors to
it.

\subsection{Under-annotated golds: examples from all three benchmarks}
\label{app:under-annotation}

The three libraries are collected from public skill repositories and
the open web, where the same capability routinely exists as more than
one document. The gold annotations name one of them, and a router
that returns a different but fully adequate skill is counted wrong by
raw Hit@1. The three examples below, one per benchmark, show what
such a miss looks like. The next subsection reports, system by system,
how far raw Hit@1 sits below the adjudicated score.

The mildest form is a library that holds the same skill twice. The
SkillRet library contains two skills named \texttt{mcp-builder},
scraped from the repositories \texttt{@xenitV1/Antigravity-Workflows}
and \texttt{@jscraik/Agent-Skills}, and the annotation lists only the
former. For the query below our router returns the other one, and the
judge prefers it in both orders, calling the annotated gold
``generic'' with ``outdated transport guidance''.

\begin{tcolorbox}[breakable, colback=black!3!white, colframe=black!65,
  fonttitle=\bfseries,
  title={SkillRet: two copies of one skill, one annotated}]
\small
\textbf{Query} (\texttt{q-00425})\\[3pt]
I want to create a server that exposes our internal ``FieldNotes''
geospatial survey dataset (stored in PostGIS) so that Claude and
other LLM-based agents can query it through the Model Context
Protocol. [\dots] Can you design the full MCP server architecture for
this --- including which capabilities should be tools vs.\ resources,
how to structure the URI scheme for individual records, input
validation best practices for the geo queries, proper error handling
patterns, and a reference implementation in TypeScript
\tcblower
\textbf{Annotated gold}\\[3pt]
name: \textbf{mcp-builder}\\
description: MCP (Model Context Protocol) server building principles.
Tool design, resource patterns, best practices.\\[6pt]
\textbf{Committed skill} (counted as a raw miss)\\[3pt]
name: \textbf{mcp-builder}\\
description: Create MCP (Model Context Protocol) servers that enable
LLMs to interact with external services through well-designed tools,
resources, and prompts. [\dots]
\end{tcolorbox}

SRA-Bench pads its 636 curated skills with 25{,}626 documents crawled
from the web, meant as distractors, and only the curated set is
annotated. The distractors are real skill documents, and some of them
serve the tasks. The task below asks for an SVM classification with
scikit-learn. Its golds are a scikit-learn reference and a guide to
Python's \texttt{warnings} module from the curated set. Our router
returns a scikit-learn skill from the distractor pool, and the judge
sides with it in all four comparisons.

\begin{tcolorbox}[breakable, colback=black!3!white, colframe=black!65,
  fonttitle=\bfseries,
  title={SRA-Bench: an adequate skill among the distractor documents}]
\small
\textbf{Query} (\texttt{bigcodebench\_00859})\\[3pt]
Perform an \textbf{SVM classification} of the iris dataset and warn if
the accuracy is less than 0.9. The warning action is set to `always'.
The test size for the train-test split is 0.33. [\dots the function
signature and output contract follow]
\tcblower
\textbf{Annotated golds}\\[3pt]
name: \textbf{scikit-learn} --- Practical Reference for Preprocessing,
Regression, Clustering, and Vectorization\\
description: Guide for correct sklearn API usage: scalers, PCA,
regression, clustering, text vectorization, and common pitfalls.\\[4pt]
name: warnings --- simplefilter, warn, glob extensions, shutil.move\\
description: Guide for Python warnings module: simplefilter placement,
conditional flags, file transfer error patterns, and distinction from
logging.warning.\\[6pt]
\textbf{Committed skill} (\texttt{web\_22564}, counted as a raw miss)\\[3pt]
name: \textbf{scikit-learn}\\
description: The industry standard library for machine learning in
Python. Provides simple and efficient tools for predictive data
analysis, covering classification, regression, clustering,
dimensionality reduction, model selection, and preprocessing.
\end{tcolorbox}

The extreme form is a verbatim copy. In Eval-Core's 78K pool the gold
skill for the task below, \texttt{rdkit}, appears a second time under
a different identifier. The two documents share their description
word for word and their bodies differ only marginally (4-gram Jaccard
$.92$). Our router returns the copy, raw Hit@1 counts a miss, and the
judge's verdict on the pair opens with ``Both documents provide
essentially identical RDKit guidance''.

\begin{tcolorbox}[breakable, colback=black!3!white, colframe=black!65,
  fonttitle=\bfseries,
  title={Eval-Core: a verbatim copy of the gold in the pool}]
\small
\textbf{Query} (\texttt{find-topk-similiar-chemicals}, sic)\\[3pt]
Find the top k similar chemicals in \texttt{molecules.pdf} to any
chemicals you are given. For converting chemical names into molecular
representations, you need to use an external chemistry resources like
PubChem or RDKit. For computing similarity, use Morgan fingerprints
with Tanimoto similarity (radius = 2, include chirality). [\dots]
\tcblower
\textbf{Annotated gold} (\texttt{gt/rdkit})\\[3pt]
name: \textbf{rdkit}\\
description: Cheminformatics toolkit for fine-grained molecular
control. SMILES/SDF parsing, descriptors (MW, LogP, TPSA),
fingerprints, substructure search, 2D/3D generation, similarity,
reactions. [\dots]\\[6pt]
\textbf{Committed skill}
(\texttt{development/rdkit-sanand0-scientific-research}, counted as a
raw miss)\\[3pt]
name: \textbf{rdkit}\\
description: identical, word for word.
\end{tcolorbox}

\subsection{Adjudication: raw scores and human agreement}
\label{app:raw-scores}

Tables~\ref{tab:raw-main}--\ref{tab:raw-traj} pair every adjudicated
score plotted in Figures~\ref{fig:main-results}, \ref{fig:ablations}
and~\ref{fig:traj} with the same system's raw Hit@1, which credits
only the annotated golds. Adjudication lifts every system, and on
every benchmark it lifts a baseline more than \textsc{Gavel}.
Progressive disclosure gains 21.1 points on SkillRet and 36.0 on
Eval-Core easy against \textsc{Gavel}'s 10.0 and 25.3, and on
SRA-Bench the 8B pipeline gains 9.4 against 5.1. The comparisons the
paper rests on do not depend on the judge either. In every column of
the three tables the raw Hit@1 of the full \textsc{Gavel} pipeline is
above every other row's, the four ablated variants included. The
judge sees only the task and the two documents, and a top-1 matching
an annotated gold is credited without a judge call.

We validate the judge against human annotation on 200 pairs, and the
two agree on the credit decision on 171 of them (86\%). The unit is
one judged pair, a task together with the non-gold skill a system
committed and one annotated gold, which the judge compared in both
orders. Aggregating the two orders sorts each pair into three
classes, depending on whether the committed skill wins more orders
than the gold, the two split or tie, or the gold wins more. We
stratify the judged pairs by benchmark and by this verdict and draw
200 of them, 50 per benchmark, balanced across the three classes as
far as the class counts allow. SkillTraj has only 14 pairs where the
judge prefers the committed skill, and the shortfall goes to the
largest class. The authors label the pairs
under the judge's own instructions. They see only the task and the two
skills in randomized order, blind to which is the annotated gold and
to what the judge said, and pick the better first load or call the two
equal. Table~\ref{tab:judge-human} cross-tabulates the two verdicts.

\begin{table}[ht]
\centering
\small
\caption{Human agreement with the adjudicating judge on 200 pairs,
stratified by benchmark and by the judge's verdict. Rows are the
judge's two-order verdict on a (task, committed skill, gold) pair;
columns are the authors' blind verdict on the same pair.}
\label{tab:judge-human}
\begin{tabular}{l c ccc c}
\toprule
& & \multicolumn{3}{c}{Human verdict} & \\
\cmidrule(lr){3-5}
Judge's verdict & $n$ & committed & equal & gold & agreement\\
\midrule
Committed skill better & 65 & 42 & 18 & 5 & .646\\
Equal & 68 & 16 & 38 & 14 & .559\\
Gold better & 67 & 3 & 7 & 57 & .851\\
\bottomrule
\end{tabular}
\end{table}

Where the two disagree they rarely point in opposite directions. On
the 132 pairs where the judge picked a side, the human label lands on
the opposite side 8 times. Of the 68 pairs the judge tied, the
annotators tie 38 and give 14 to the gold. Since the crediting
rule counts a tie as credit, those 14 are credited against the human
label. In all, the annotators would withhold credit on 19 pairs the
judge credited and grant it on 10 the judge withheld, a lean toward
generosity of 9 pairs in 200. We treat the human labels as a check on
the judge and no more. A skill document runs to thousands of tokens
of instructions, scripts, and reference material, and whether it
serves a task can turn on a detail deep in the body that an annotator
reads less closely than the model does. A disagreement therefore does
not by itself say which side erred.

\begin{table}[htbp]
\centering
\small
\setlength{\tabcolsep}{3.5pt}
\caption{Raw Hit@1 beside the adjudicated scores of
Figure~\ref{fig:main-results}. Indented rows are the retrieval stage
committing its top-1 alone, the light bar segments of the figure.}
\label{tab:raw-main}
\begin{tabular}{l cc cc cc cc}
\toprule
& \multicolumn{2}{c}{SkillRet test} & \multicolumn{2}{c}{SRA-Bench}
& \multicolumn{2}{c}{Eval-Core easy} & \multicolumn{2}{c}{Eval-Core hard}\\
\cmidrule(lr){2-3}\cmidrule(lr){4-5}\cmidrule(lr){6-7}\cmidrule(lr){8-9}
& raw & adj. & raw & adj. & raw & adj. & raw & adj.\\
\midrule
BM25 & .3786 & .5007 & .4820 & .5226 & .5600 & .6667 & .5200 & .6400\\
SKILLRET-Emb-0.6B & .7699 & .8797 & .4204 & .4855 & .6267 & .8000 & .5733 & .7333\\
\addlinespace
Qwen3-Emb-8B + prog.\ discl. & .6526 & .8635 & .5633 & .6434 & .5467 & .9067 & .5467 & .8933\\
\quad\emph{retrieval stage alone} & .6540 & .8273 & .5528 & .6074 & .5600 & .8533 & .5200 & .8000\\
\addlinespace
SkillRouter-Emb-0.6B + RR-0.6B & .8301 & .9077 & .4135 & .4866 & .6667 & .8267 & .6133 & .7867\\
\quad\emph{retrieval stage alone} & .7416 & .8427 & .2462 & .3031 & .6000 & .8133 & .5333 & .7600\\
\addlinespace
Qwen3-Emb-8B + Qwen3-RR-8B & .6918 & .8645 & .6458 & .7398 & .6667 & .9200 & .6533 & .9200\\
\quad\emph{retrieval stage alone} & .6540 & .8273 & .5528 & .6074 & .5600 & .8533 & .5200 & .8000\\
\addlinespace
\textbf{\textsc{Gavel} (ours)} & .8459 & .9460 & .8223 & .8734 & .6933 & .9467 & .6933 & .9333\\
\quad\emph{retrieval stage alone} & .7681 & .8725 & .6272 & .6957 & .6800 & .8933 & .6400 & .8667\\
\bottomrule
\end{tabular}
\end{table}

\begin{table}[htbp]
\centering
\small
\setlength{\tabcolsep}{3.5pt}
\caption{Raw Hit@1 beside the adjudicated scores of
Figure~\ref{fig:ablations}. Every row is the full pipeline with one
design choice replaced.}
\label{tab:raw-ablations}
\begin{tabular}{l cc cc cc cc}
\toprule
& \multicolumn{2}{c}{SkillRet test} & \multicolumn{2}{c}{SRA-Bench}
& \multicolumn{2}{c}{Eval-Core easy} & \multicolumn{2}{c}{Eval-Core hard}\\
\cmidrule(lr){2-3}\cmidrule(lr){4-5}\cmidrule(lr){6-7}\cmidrule(lr){8-9}
& raw & adj. & raw & adj. & raw & adj. & raw & adj.\\
\midrule
\textbf{\textsc{Gavel} (reference)} & .8459 & .9460 & .8223 & .8734 & .6933 & .9467 & .6933 & .9333\\
\addlinespace
Glance $\to$ Jina-ColBERT-v2 & .5968 & .7803 & .6690 & .7096 & .6400 & .8533 & .6400 & .8400\\
Product of experts $\to$ pointwise rerank & .4068 & .7170 & .6109 & .7201 & .5333 & .8267 & .5333 & .8400\\
Verdict $\to$ skill body in generation & .7835 & .9081 & .7108 & .7782 & .6800 & .8400 & .6800 & .8400\\
Verdict $\to$ progressive disclosure & .7354 & .9001 & .5958 & .6945 & .6000 & .8800 & .6000 & .8667\\
\bottomrule
\end{tabular}
\end{table}

\begin{table}[htbp]
\centering
\small
\setlength{\tabcolsep}{3.5pt}
\caption{Raw Hit@1 beside the adjudicated scores of
Figure~\ref{fig:traj}, per scenario.}
\label{tab:raw-traj}
\begin{tabular}{l cc cc cc cc}
\toprule
& \multicolumn{2}{c}{user} & \multicolumn{2}{c}{tool}
& \multicolumn{2}{c}{agent} & \multicolumn{2}{c}{wrong-skill}\\
& \multicolumn{2}{c}{request} & \multicolumn{2}{c}{evidence}
& \multicolumn{2}{c}{plan} & \multicolumn{2}{c}{recovery}\\
\cmidrule(lr){2-3}\cmidrule(lr){4-5}\cmidrule(lr){6-7}\cmidrule(lr){8-9}
& raw & adj. & raw & adj. & raw & adj. & raw & adj.\\
\midrule
SkillRouter-Emb-0.6B + RR-0.6B\\
\quad\emph{full context} & .4741 & .4828 & .1132 & .1226 & .2095 & .2095 & .2000 & .2000\\
\quad\emph{last message} & .8621 & .8707 & .3019 & .3396 & .4286 & .4381 & .6667 & .6667\\
\addlinespace
Qwen3-Emb-8B + Qwen3-RR-8B\\
\quad\emph{full context} & .6897 & .7069 & .4057 & .4434 & .4762 & .4952 & .5111 & .5111\\
\quad\emph{last message} & .8707 & .8793 & .4717 & .5094 & .6762 & .6857 & .6667 & .6667\\
\addlinespace
Glance + prog.\ disclosure & .7500 & .7586 & .4528 & .5283 & .1429 & .1524 & .7556 & .7778\\
\textbf{\textsc{Gavel} (ours)} & .9483 & .9655 & .6509 & .7264 & .8476 & .9048 & .8667 & .8889\\
\bottomrule
\end{tabular}
\end{table}

\subsection{Candidate-pool recall of the retrieval stages}
\label{app:pool-recall}

Table~\ref{tab:pool-recall} reports every retrieval front end of
Figures~\ref{fig:main-results} and~\ref{fig:ablations} on its own: raw
Hit@1 of its top-1, and R@20, the fraction of queries whose top-20
candidate pool contains an annotated gold. A second stage only reorders
this pool, so R@20 caps the raw Hit@1 of any pipeline built on the
front end. Both columns credit annotated golds only, and outside
Eval-Core, where 75 queries leave more room for chance, they order the
front ends largely the same way.

\begin{table}[htbp]
\centering
\small
\setlength{\tabcolsep}{3.5pt}
\caption{Raw Hit@1 and top-20 candidate-pool recall of the retrieval
stages behind Figures~\ref{fig:main-results} and~\ref{fig:ablations}.}
\label{tab:pool-recall}
\begin{tabular}{l cc cc cc cc}
\toprule
& \multicolumn{2}{c}{SkillRet test} & \multicolumn{2}{c}{SRA-Bench}
& \multicolumn{2}{c}{Eval-Core easy} & \multicolumn{2}{c}{Eval-Core hard}\\
\cmidrule(lr){2-3}\cmidrule(lr){4-5}\cmidrule(lr){6-7}\cmidrule(lr){8-9}
& Hit@1 & R@20 & Hit@1 & R@20 & Hit@1 & R@20 & Hit@1 & R@20\\
\midrule
BM25 & .3786 & .6410 & .4820 & .8235 & .5600 & .7467 & .5200 & .7467\\
Jina-ColBERT-v2 & .4156 & .7226 & .4483 & .7445 & .4667 & .7867 & .4533 & .7867\\
SKILLRET-Emb-0.6B & .7699 & .9680 & .4204 & .7008 & .6267 & .8933 & .5733 & .8933\\
SkillRouter-Emb-0.6B & .7416 & .9490 & .2462 & .6225 & .6000 & .8133 & .5333 & .8000\\
Qwen3-Emb-8B & .6540 & .9172 & .5528 & .9117 & .5600 & .8533 & .5200 & .8533\\
\addlinespace
\textbf{Glance (ours)} & .7681 & .9706 & .6272 & .9675 & .6800 & .8800 & .6400 & .8533\\
\bottomrule
\end{tabular}
\end{table}

\subsection{SkillTraj: construction and examples}
\label{app:skilltraj}

Each trajectory starts from its gold skill, drawn from the held-out
SkillRet library of 6{,}660 with at most two trajectories per skill.
GPT-5.6 Sol then writes the full dialogue against a scenario brief. The
setting is a plain tool-using assistant whose system prompt lists no
skills; a trajectory spans two to eight user turns and up to five tool
calls with their results. Alongside the dialogue the generator returns machine-checkable anchors. One is the marked routing point, and when that point sits inside an assistant message, the other is the verbatim sentence that creates it. Both are verified programmatically. The sentence must
appear in the message exactly, and the prefix up to the routing point
must tokenize identically to the full rendering, so the point is a
real token position of the live context.

Admission is judged blind. The judge, Claude Sonnet 4.6, sees only
prefixes truncated at the point under test, never the continuation, and the two
views of a before-versus-after check are judged independently, in
random order. The bar is the one stated in the main text. The need must exceed what a generic assistant can serve, or depend on the skill's specific mechanisms. The gold must also solve it when shown among its nine nearest library neighbors with identities hidden. If the judge finds a neighbor that serves the need as well, the trajectory is dropped or the neighbor is recorded as a second gold.
The leakage screen bans the gold's name and its variants outright. It also runs $n$-gram and longest-common-substring checks against the gold's body over every message, tool argument, and tool result, and sends trajectories with long verbatim overlap to a blind audit. A
trajectory that fails any check is regenerated from scratch, never
patched. The rest of this subsection walks the four scenarios, each
with an excerpt from a real trajectory around its routing point.

\textbf{User request.} The need arrives inside a user message and the
routing point closes that turn. Across trajectories that turn ranges from the opening request to late in the dialogue. In the example below it lands
three exchanges in, after the dialogue has built up context that has
nothing to do with the skill:

\begin{tcolorbox}[colback=black!3!white, colframe=black!60,
  fonttitle=\bfseries, title={User request}]
\textbf{User:} I'm organizing a two-day staff retreat for a 12-person
nonprofit team. [...] Can you suggest a theme and a short tagline that
feel warm but not cheesy?

{\color{gray}\emph{[The assistant answers; two more exchanges polish
the invitation and sketch the run-of-show.]}}

\textbf{User:} Now build the actual Obsidian .canvas file for this
retreat, not a Mermaid diagram or an image. Return only the file
contents. Use the native nodes-and-edges object structure, with
separate group objects for Friday and Saturday placed before their
contained items so the backgrounds stay underneath. [...] Connect the
agenda in chronological order with side-specific anchors and
arrowheads, including the overnight transition from Friday to
Saturday. {\color{gray}$\triangleleft$\,\emph{routing point}}
\tcblower
\textbf{Gold skill} \texttt{json-canvas} --- Create and edit JSON
Canvas files (.canvas) with nodes, edges, groups, and connections. Use
when working with .canvas files, creating visual canvases, mind maps,
flowcharts, or when the user mentions Canvas files in Obsidian.
\end{tcolorbox}

\textbf{Tool evidence.} A tool result reveals the need. The generator
quotes the revealing span verbatim, and the judge reads the trajectory
truncated just before and just after the result; the need must be
absent in the first reading and present in the second. In the example, the user's request to diagnose a failed benchmark run could still be ordinary debugging, and only the fetched report reveals a hard-exploration problem:

\begin{tcolorbox}[colback=black!3!white, colframe=black!60,
  fonttitle=\bfseries, title={Tool evidence}]
{\color{gray}\emph{[Two earlier exchanges cover the repository's
setup and a passing test run.]}}

\textbf{User:} The overnight benchmark report is published at
\nolinkurl{metrics.example.net/runs/nightly-1842}. Diagnose why it missed
the target, implement a concrete fix in this repository, and validate
the change.

\textbf{Assistant:} I'll retrieve the benchmark record, then trace the
relevant training path and apply the fix.
{\color{gray}\emph{[calls]}} \texttt{fetch\_url}

\textbf{Tool:} Nightly run 1842 [...] Benchmark:
Labyrinth-Key-Image-v1. Budget: 2,000,000 environment steps. Target:
success\_rate $\geq$ 0.25. Result: success\_rate = 0.00. External
reward: 0 on every transition except +1 after collecting a hidden key
and reaching the locked exit. [...]
{\color{gray}$\triangleleft$\,\emph{routing point}}
\tcblower
\textbf{Gold skill} \texttt{exploration-strategies} --- Master
$\varepsilon$-greedy, UCB, curiosity-driven, RND, intrinsic motivation
exploration.
\end{tcolorbox}

\textbf{Agent plan.} No new user or tool information intervenes. The agent's own working text spells out the sub-task in one sentence, and
the routing point is that sentence's final token, in the middle of the
message. The truncations at the sentence's two ends are judged as in
tool evidence. In the example nothing the user asks calls for a skill
on its own; the need appears when the agent decides its draft should
fold in the team's recorded lessons:

\begin{tcolorbox}[colback=black!3!white, colframe=black!60,
  fonttitle=\bfseries, title={Agent plan}]
\textbf{User:} Help me put together a 45-minute retrospective agenda
for our billing migration rollout.

{\color{gray}\emph{[The assistant proposes a structure and asks about
attendees and tone.]}}

\textbf{User:} It'll be eight engineers, the PM, and our support
lead. Keep it constructive and fairly direct. The rollout finished two
weeks late, but there were no customer-facing outages.

\textbf{Assistant:} That group is small enough for a structured
discussion [...] I'll keep the opening neutral, spend more time on the
delay than on celebration, and end with a short list of owned actions.
Before I finalize it, I should search the Markdown files in the
vault's Areas/AI/Memory directory by both note text and frontmatter
concept tags, then prioritize any billing-migration matches by
importance so the retrospective incorporates the team's previously
recorded lessons. {\color{gray}$\triangleleft$\,\emph{routing point}}
[...]
\tcblower
\textbf{Gold skill} \texttt{recall} --- Search Obsidian memories by
keyword and concept.
\end{tcolorbox}

\textbf{Wrong-skill recovery.} The user establishes the need, but the
assistant misdiagnoses it and opens the SKILL.md of a wrong skill, whose body enters the context verbatim as the tool result. The wrong skill is sampled from the gold's nearest library neighbors, excluding near-copies of the gold. The assistant tries
to apply it, fails, and writes a recovery sentence naming the
capability it actually needs, which is the routing point. The judge
additionally confirms that the document already in context cannot
serve the need:

\begin{tcolorbox}[colback=black!3!white, colframe=black!60,
  fonttitle=\bfseries, title={Wrong-skill recovery}]
{\color{gray}\emph{[An earlier exchange names a git branch for the
change at hand.]}}

\textbf{User:} Before I add a background job that retries failed
invoice exports, find the strongest existing example in our codebase.
I need the exact file and line references, the standard implementation
flow and conventions it demonstrates, common pitfalls, and the
matching testing approach so I can follow the project's established
pattern.

\textbf{Assistant:} I'll focus specifically on an established
background-job example rather than giving generic queueing advice.
[...] {\color{gray}\emph{[loads the skill]}}
\texttt{anti-duplication}

\textbf{Loaded skill:} name: anti-duplication [...] description:
Before implementing new code (endpoints, components, services,
models), search the codebase for existing patterns to reuse. [...]

\textbf{Assistant:} The guidance I opened is aimed at preventing
duplicate implementations [...] That does not identify and rank the
canonical background-job example [...] What is actually needed is to
locate the codebase's canonical background-job implementation and
return exact line references, its standard flow, conventions,
pitfalls, and matching test examples.
{\color{gray}$\triangleleft$\,\emph{routing point}} [...]
\tcblower
\textbf{Gold skill} \texttt{pattern-finder} --- Find implementation
examples and design patterns in the codebase. Use when user asks ``how
to implement X'', ``how does this project do X'', ``show me examples
of X'', ``where is X implemented'', or needs to follow existing code
conventions.
\end{tcolorbox}

% TODO (rest of app:benchmarks): remaining disclosure items.

%% file: appendix/09_deployment.tex
\section{Deployment walkthroughs}
\label{app:harness}

This appendix walks through transcripts from the Skill-Use rollouts
behind Table~\ref{tab:skilluse}, run in
mini-swe-agent\footnote{\url{https://github.com/SWE-agent/mini-swe-agent}}
and reproduced verbatim. The frontier models' rows in that table come
from the Skill-Use paper, which reports each model in both Codex and Claude Code. We quote the Codex numbers, which are the higher of the two for every model. Of the four, DeepSeek-V4-Pro (1.6T total and 49B active
parameters), GLM-5.1 (754B and 40B) and MiniMax-M3 are released as
open-weight mixture-of-experts models, while Qwen3.6-Max is served only
through an API and its size is not disclosed. The gate is read
at every decoded token, so a firing interrupts the message wherever it
happens to stand. On a load, the half-written message stays in the
context exactly as far as it got, the chosen skill's full SKILL.md is
appended as a \texttt{[skill loaded: $\ldots$]} message, and the
assistant turn starts over with the skill in view. Across the 177 tasks the router commits 181 loads. Most fire within the opening sentence of the first response, and seventeen fire in a later turn, after command output has entered the context. The interrupted position is
marked {\color{gray}$\triangleleft$\,\emph{gate fires}}.

In the first transcript the agent sets out on its own, even guessing
the right tool. The opening attempt fails on a wrong path, and the
gate stays quiet through the search that follows. It fires the moment the agent reports that it has found the form and turns to filling it in:

\begin{tcolorbox}[colback=black!3!white, colframe=black!60,
  fonttitle=\bfseries, title={Filling a W-4 form}]
\textbf{User:} A new employee onboarding needs to fill out the IRS W-4
form. The form template is at \texttt{forms/fw4.pdf} [...] I have
already prepared the employee information FDF data file
\texttt{employee\_w4.fdf}. Please use this data to fill the W-4 form.
Note: after filling, the employee still needs to sign and confirm
themselves, so do not lock the form; keep it editable. [...]

\textbf{Assistant:} To fill the W-4 form using the provided FDF data
file, I will use the \texttt{pdftk} tool. [...]\\
\texttt{pdftk forms/fw4.pdf fill\_form employee\_w4.fdf output
fw4\_filled.pdf}

\textbf{Output:} Error: Unable to find file. Error: Failed to open
input PDF file: \texttt{forms/fw4.pdf}

{\color{gray}\emph{[An \texttt{ls forms/} fails the same way; a
\texttt{find} then locates the form. The gate stays quiet through all
three observations.]}}

\textbf{Assistant:} I have located the \texttt{fw4.pdf} file in the
\texttt{./assets/SF} {\color{gray}$\triangleleft$\,\emph{gate fires}}
\tcblower
\textbf{Loaded} \texttt{pdftk-server} (gold) --- ``Skill for using the
command-line tool pdftk (PDFtk Server) for working with PDF files. Use
when asked to [...] fill PDF forms [...]'' The restarted turn
completes the interrupted sentence, re-issues \texttt{fill\_form} on
the found path, and the filled form is saved unflattened, editable as
the task requires.
\end{tcolorbox}

In the next transcript the gate fires as the reply is about to name
the need:

\begin{tcolorbox}[colback=black!3!white, colframe=black!60,
  fonttitle=\bfseries, title={Migrating an endpoint to a new SDK}]
\textbf{User:} The product review endpoint in
\texttt{[...]/review-service/} needs to be updated to follow the
current Vercel AI SDK v6 implementation patterns. Preserve the
existing validation behavior and keep the caller-facing JSON contract
unchanged.

\textbf{Assistant:} I need to examine the current implementation of
the product review endpoint to understand what changes are required to
align it with the {\color{gray}$\triangleleft$\,\emph{gate fires}}
\tcblower
\textbf{Loaded} \texttt{vercel-ai-sdk} (gold) --- ``Guide for Vercel
AI SDK v6 implementation patterns including generateText, streamText,
[...] structured output with Output helpers [...]'' The next turn
finishes the sentence, ``\ldots to align it with the Vercel AI SDK v6
patterns,'' and the eventual edit is the one the skill prescribes,
replacing the v5 \texttt{generateObject} with \texttt{generateText}
plus an \texttt{Output.object} schema.
\end{tcolorbox}

The deepest fire of the run comes in the third turn, after the agent has read an instruction file and dumped the form's text. The gate fires a hundred-plus tokens into the message that turns that evidence into edits, midway through an email address:

\begin{tcolorbox}[colback=black!3!white, colframe=black!60,
  fonttitle=\bfseries, title={Editing a PDF form after the evidence is
  in}]
\textbf{User:} You are given a PDF file at [...] and a text file at
[...] containing natural language instructions and updated
information. Your task is to update this PDF using the latest
information provided in the text file and follow all the instructions
in it. [...]

{\color{gray}\emph{[The agent reads the instruction file, which holds a student's corrected personal information and an appeal reason, then dumps the form, a university insurance-waiver appeal filled with a nickname, a wrong email, and a wrong ID.]}}

\textbf{Assistant:} I now see the structure of the PDF form. The task
requires me to update the following fields with the information from
the text file: 1. Student Name: Replace ``Yaya'' with ``Jinya
Jiang''. 2. Student PID\#: Redact ``A88888888'' to show only the last
4 digits [...] 3. UCSD E-mail: Update ``jiang@gmail.com'' to
``jiang@ {\color{gray}$\triangleleft$\,\emph{gate fires}}
\tcblower
\textbf{Loaded} \texttt{pdf-editing-14} (gold) --- ``Complete guide
for reading and editing PDF documents with PyMuPDF.'' The restarted
turn folds the checklist into a PyMuPDF script that follows the skill's rules, replacing text at its original position and redacting with a white fill, and submits the edited form.
\end{tcolorbox}

The last transcript has two fires on one task. The gold loads on the first decoded word. Two commands later, the moment the connector's OAuth configuration scrolls into the context, the gate fires again and loads a second skill the task also needs, though it is not the marked gold:

\begin{tcolorbox}[colback=black!3!white, colframe=black!60,
  fonttitle=\bfseries, title={Hardening a connector, two loads}]
\textbf{User:} Our partner team wants to submit the Salesforce
connector package in [...] for a managed pilot. Sign-in works during
maker testing, but the security reviewer says the current
authentication flow and submission package are not sufficient for
agent-mediated use. Please harden the package [...] including the
reviewer handoff artifacts and a concise security note.

\textbf{Assistant:} I {\color{gray}$\triangleleft$\,\emph{gate
fires}}

\textbf{Loaded} \texttt{power-platform-mcp-connector-suite} (gold)
--- ``Generate complete Power Platform custom connector with MCP
integration for Copilot Studio --- includes schema generation,
troubleshooting, and validation.''

{\color{gray}\emph{[An \texttt{ls} surveys the package; a dump of
\texttt{apiDefinition.swagger.json} ends on the connector's OAuth
security block.]}}

\textbf{Assistant:} The {\color{gray}$\triangleleft$\,\emph{gate
fires again}}

\textbf{Loaded} \texttt{api-security-best-practices} (not the gold)
--- ``Implement secure API design patterns including authentication,
authorization, input validation, rate limiting, and protection
against common API vulnerabilities.'' This is what hardening the authentication flow calls for. The rollout goes on to work through
the package and ships the hardened connector with the security note
the task asks for.
\end{tcolorbox}

The progressive-disclosure row of Table~\ref{tab:skilluse} runs
Qwen3-32B in the same harness with no router. The skills enter the
way the convention prescribes, through the system prompt below, and
each skill's full SKILL.md sits on disk at the path the prompt names,
one \texttt{cat} away.

\begin{tcolorbox}[colback=blue!5!white, colframe=blue!75!black,
  fonttitle=\bfseries, title={Progressive disclosure in the harness:
  the baseline system prompt}]
\small
You are a helpful assistant that can interact with a computer.

Your response must contain exactly ONE bash code block with ONE
command (or commands connected with \&\& or \textbar\textbar).\\
Explain your reasoning before your command.

{\color{gray}\emph{[the harness's standard formatting example, rules,
and command examples, unchanged from mini-swe-agent]}}

\#\# Available skills

The following agent skills are installed. Each one's full
instructions live at \texttt{skills/<name>/SKILL.md}, which you can
read with \texttt{cat}. If a skill applies to the task, read its
SKILL.md first and follow it.

- \{name\}: \{description\}\\
$\cdots$
\end{tcolorbox}

Under this prompt the model reads a skill on 2 of the 175 tasks that
call for one. The listing itself serves the model well enough, since
forcing its first action to read the one skill it picks from that
listing raises the score to .834. The failure sits in the unprompted decision to consult a skill at all, a habit that models acquire in training for a particular harness. Even the frontier models of Table~\ref{tab:skilluse},
post-trained amid the skills ecosystem, swing by up to 38 points in
trigger rate between the two harnesses the benchmark measures \citep{han2026skilluse}. Qwen3-32B predates that ecosystem
altogether \citep{yang2025qwen3,anthropic2025skills}, so nothing in
its training rewards pausing mid-task to read installed
documentation. \textsc{Gavel} supplies the trigger from outside the
model's habits, and the same backbone reaches .909.

%% file: appendix/10_cost.tex
\section{The cost of routing}
\label{app:cost}

We price what one session spends on skill routing. Let the session
make $T$ model calls, let the library hold $n$ skills, and let the
gate fire $m$ times, where each firing triggers one \textsc{Gavel} attempt to find a
suitable skill for the current state. A skill carries roughly
100 tokens of metadata and a 1{,}895-token body, the mean over
40{,}285 publicly listed skills \citep{ling2026agentskills}; a query
carries 239 tokens, the mean over the SkillRet test set. Prices
are GPT-5.6 Sol's official rates\footnote{%
\url{https://platform.openai.com/docs/pricing}, August 2026: \$4.00
per million input tokens, \$0.40 per million cached reads, \$5.00 per
million cache writes. We ignore the long-context surcharge on
requests beyond 272K input tokens.}. Neither router decodes extra tokens, and both
inject the chosen body after routing, so those costs cancel. Budgeting
the injected bodies themselves is a separate problem \citep{chen2026bps}.

The gate and the glance read hidden states the rollout computes
anyway, and the key bank is built offline once per library, so \textsc{Gavel} pays only for verdicts. Under the pruning margin of Section~\ref{sec:exp-setup} a firing runs nine verdict forwards, each prefilling
one candidate's metadata and body followed by the query:
\begin{equation*}
C_{\textsc{Gavel}} \;=\; m \times 9 \times (100+1{,}895+239)
\times \$4/10^{6} \;\approx\; \$0.080\,m .
\end{equation*}
The bill is independent of $n$ and $T$, and it is paid off-path, since the forwards are batched prefills that never enter the rollout's context.
Progressive disclosure keeps its $100n$-token metadata menu in the prompt prefix. The serving API is stateless, so every one of the $T$ calls re-sends the context, and ``all previous input tokens \ldots
are billed as input tokens'' each time \citep{openai2026convstate}.
The menu is therefore written to the prompt cache once and read back
at the cached rate on the other $T-1$ calls:
\begin{equation*}
C_{\text{PD}} \;=\; 100n \times \bigl(\$5 + \$0.40\,(T-1)\bigr)/10^{6}
\;\approx\; \$4\times10^{-5}\; nT .
\end{equation*}
At realistic scales both bills are small and close. Three firings cost \$0.24, and a hundred-skill menu over fifty calls costs \$0.25. In general, \textsc{Gavel} is the cheaper router once $n$ exceeds roughly $2{,}000\,m/T$.

Latency follows the same shape. The gate reads a state the rollout
has already computed, and the glance is one matrix multiply of the
projected query tokens against the bank, milliseconds on a GPU. The wait comes from the verdict, whose nine prefills of about 2.2K tokens each run as one batch and take on the order of half a second per firing on a modern serving stack, less once the per-skill prefixes are KV-cached.
That pause is paid once per firing. Progressive disclosure adds no
separate stage, but every decoded token of the session attends over
the $100n$ menu positions. The prompt cache spares their recompute but not the attention over them at every step.

The menu's real cost is the window it occupies. It holds $100n$
tokens of the model's finite context for the whole session, and
metadata held in front of every call disperses the model's attention
and degrades its performance on the task itself
\citep{liu2024lost,modarressi2025nolima}. A verdict forward never enters that window, since the nine prefills run off-path. Their skill prefix is also query-independent, so a serving stack can cache it per
skill and pay fresh input only for the live query.

%% file: appendix/11_discussion.tex
\section{Discussion}
\label{app:discussion}

\paragraph{The gate.} The gate of Section~\ref{sec:exp-harness} is a
linear classifier over two read-outs the rollout already produces, namely the glance scores that the context has accumulated over the library and the final-layer state at the current position. It fires once it stays on for two consecutive tokens. We trained it once on
SkillTraj and did not engineer it further, since the question
Section~\ref{sec:exp-harness} asks is whether the routing signal
survives a live rollout, not how precisely a call can be timed. Two
designs would time it better. A sequence model over the history of
glance scores would watch the need build up across positions instead
of judging each one alone, and a backbone fine-tuned to emit a
dedicated skill-call token would fold the decision into decoding
itself. Neither changes the glance or the verdict, since the gate only decides when they run, and we leave both to future work.

\paragraph{Routing behind a cloud API.} Our experiments run the
backbone locally, while many agents call it through a hosted API.
\textsc{Gavel} moves there by letting the provider host the router.
The two projections are trained once per model on a corpus of
task--skill pairs with a contrastive loss, and
shipped with the model the way a tokenizer is. A client uploads its
skills, and the provider builds the compressed bank from them in one
forward pass each, as it would write a prompt cache. On each call the
glance runs inside the prefill the provider already performs, reading
the hidden states beside the KV cache it keeps for those tokens. The verdict then prefills the shortlisted skills off-path, with each skill prefix served from the prompt cache. The response carries the chosen
skill, or nothing when the verdict abstains, and the gate runs in the
same place from the same states. Nothing crosses the API beyond the
skills at installation and the routing decision on each call.